\documentclass{tlp}
\usepackage{aopmath}
\usepackage{algorithm}
\usepackage{amssymb}
\usepackage{algpseudocode}
\usepackage{tikz}
\usepackage{multicol}
\usepackage{subfig}
\usepackage{listings}
\usepackage{url}
\newtheorem{definition}{Definition} 
\newtheorem{example}{Example} 

\providecommand{\codeif}{\texttt{:- }}
\providecommand{\ruleend}{\texttt{.}}

\providecommand{\naf}[1]{\texttt{not }#1}
\providecommand{\asp}[1]{\mbox{$\mathtt{#1}$}}
\providecommand{\algname}{ILASP2i }

\begin{document}

\title{Iterative Learning of Answer Set Programs from Context Dependent Examples}
\shorttitle{Iterative learning of answer set programs from context-dependent examples}

\author[M. Law, A. Russo, K. Broda]{ Mark Law, Alessandra Russo\thanks{This
  research is partially funded by the EPSRC project EP/K033522/1 ``Privacy
Dynamics''.}, Krysia Broda \\ Department of Computing, Imperial College London,
SW7 2AZ\\ \email{$\lbrace$mark.law09, a.russo, k.broda$\rbrace$@imperial.ac.uk }}


\maketitle

\begin{abstract}

In recent years, several frameworks and systems have been proposed that extend
Inductive Logic Programming (ILP) to the Answer Set Programming (ASP) paradigm.
In ILP, examples must all be explained by a hypothesis together with a given
background knowledge. In existing systems, the background knowledge is the same
for all examples; however, examples may be context-dependent. This means that
some examples should be explained in the context of some information, whereas
others should be explained in different contexts. In this paper, we capture
this notion and present a context-dependent extension of the \emph{Learning
from Ordered Answer Sets} framework.  In this extension, contexts can be used
to further structure the background knowledge. We then propose a new iterative
algorithm, ILASP2i, which exploits this feature to scale up the existing ILASP2 system
to learning tasks with large numbers of examples.  We demonstrate the gain in
scalability by applying both algorithms to various learning tasks.  Our results
show that, compared to ILASP2, the newly proposed ILASP2i system can be two
orders of magnitude faster and use two orders of magnitude less memory, whilst
preserving the same average accuracy.
This paper is under consideration for acceptance in TPLP.

\end{abstract}

\section{Introduction}
\label{sec:intro}

Inductive Logic Programming~\cite{muggleton1991} (ILP) addresses the task of
learning a logic program, called a {\em hypothesis}, that explains a set of
examples using some background knowledge. Although ILP has traditionally
addressed learning (monotonic) definite logic programs, recently, several new
systems have been proposed for learning under the (non-monotonic) answer set
semantics (e.g.~\cite{ray2009nonmonotonic}, \cite{Corapi2012}, \cite{raspal},
\cite{JELIA_ILASP} and \cite{ICLP15}). Among these, ILASP2~\cite{ICLP15}
extended ILP to \emph{learning from ordered answer sets} ($ILP_{LOAS}$), a
computational task that learns answer set programs containing normal rules,
choice rules and both hard and weak constraints.

Common to all ILP systems is the underlying assumption that hypotheses should
cover the examples with respect to one fixed given background knowledge.  But,
in practice, some examples may be context-dependent -- different examples may
need to be covered using different background knowledges. For instance, within
the problem domain of urban mobility, the task of learning journey preferences
of people in a city may require a general background knowledge that describes
the different modes of transport available to a user (walk, drive, etc.), and
examples of which modes of transport users choose for particular journeys. In
this case, the context of an example would be the attributes (e.g. the
distance) of the journey. It is infeasible to assume that every possible
journey could be encoded in the background knowledge -- attributes, such as
journey distances, may take too many possible values.  But, encoding the
attributes of observed journeys as contexts of the observations restricts the
computation to those attribute values that are in the contexts.

In this paper, we present a generalisation of $ILP_{LOAS}$, called
\emph{context-dependent learning from ordered answer sets}
($ILP_{LOAS}^{context}$), which uses \emph{context-dependent} examples.  We
show that any $ILP_{LOAS}^{context}$ task can be translated into an
$ILP_{LOAS}$ task, and can therefore be solved by ILASP2. Furthermore, to
improve the scalability of ILASP2, we present a new iterative reformulation of
this learning algorithm, called  \emph{ILASP2i}.  This iterative approach
differs from existing non-monotonic learning systems, which tend to be batch
learners, meaning that they consider all examples at once. Non-monotonic
systems  cannot use a traditional cover loop
(e.g.,~\cite{muggleton1995inverse}), as examples that were covered in previous
iterations are not guaranteed to be covered in later iterations.  However,
ILASP2i iteratively computes a hypothesis by constructing a set of examples
that are \emph{relevant} to the search, without the need to consider all
examples at once.  Relevant examples are essentially counterexamples for the
hypotheses found in previous iterations.  This approach is a middle ground
between batch learning and the cover loop: it avoids using the whole set of
examples, but works in the non-monotonic case, as the relevant examples persist
through the iterations. We show that ILASP2i performs significantly better than
ILASP2 in solving learning from ordered answer set tasks with large numbers of
examples, and better still when learning with context-dependent examples, as in
each iteration it only considers the contexts of relevant examples, rather than
the full set.

To demonstrate the increase in scalability we compare ILASP2i to ILASP2 on a
variety of tasks from different problem domains. The results show that ILASP2i
is up to 2 orders of magnitude faster and uses up to 2 orders of magnitude less
memory than ILASP2. We have also applied both algorithms to the real-world
problem domain of urban mobility, and explored in greater depth the task of
learning a user's journey preferences from pairwise examples of which journeys
are preferred to others. As we learn ASP, these user preferences can very
naturally be represented as weak constraints, which give an ordering over the
journeys. Our results show that ILASP2i achieves an accuracy of at least $85\%$
with around 40 examples.  We also show that, by further extending
$ILP_{LOAS}^{context}$ with ordering examples that express equal preferences,
in addition to strict ordering, the accuracy can increase to $93\%$.

The rest of the paper is structured as follows. In Section~\ref{sec:background}
we review the relevant background. In Section~\ref{sec:approach} we present our
new context-dependent learning from ordered answer set task, and in
Section~\ref{sec:algorithm} we introduce our new ILASP2i algorithm. In
Section~\ref{sec:benchmarks} we compare ILASP2i to ILASP2 on a range of
different learning tasks and give a detailed evaluation of the accuracy of
\algname and compare its scalability with ILASP2 in the context of the journey
planning problem.  Finally, we conclude the paper with a discussion of related
and future work.

\section{Background}
\label{sec:background}


Let $\asp{h},\asp{h_1},\ldots,\asp{h_k},\asp{b_1},\ldots,\asp{b_n}$ be atoms
and $\asp{l}$ and $\asp{u}$ be integers.  The ASP programs we consider contain
normal rules, of the form $\asp{h \codeif b_1,\ldots, b_m, \naf b_{m+1},
\ldots, \naf b_n}$; constraints, which are rules of the form $\asp{\codeif
b_1,\ldots, b_m,\naf b_{m+1}, \ldots, \naf b_n}$; and choice rules, of the form
$\asp{l \lbrace h_1, \ldots, h_k\rbrace u \codeif b_1,\ldots, b_m,\naf b_{m+1},
\ldots, \naf b_n}$. We refer to the part of the rule before the ``$\texttt{:-}$''
as the head, and the part after the ``$\texttt{:-}$'' as the body. The meaning of
a rule is that if the body is true, then the head must be true. The empty head
of a constraint means \emph{false}, and constraints are used to rule out answer
sets. The head of a choice rule is true if between $\asp{l}$ and $\asp{u}$
atoms from $\asp{h_1,\ldots,h_k}$ are true.  The solutions of an ASP program
$P$ form a subset of the Herbrand models of $P$, called the \emph{answer sets}
of $P$ and denoted as $AS(P)$.

ASP also allows optimisation over the answer sets according to \emph{weak
constraints}, which are rules of the form $\asp{:\sim}$ $\asp{
b_1,}$ $\asp{ \ldots,}$ $\asp{ b_m,}$ $\asp{ \naf b_{m+1},}$ $\asp{ \ldots,}$
$\asp{\naf b_{n}\ruleend[w@p,}$ $\asp{t_1,}$ $\asp{\ldots,}$ $\asp{t_k]}$ where
$\asp{b_1,\ldots, b_n}$ are atoms called (collectively) the \emph{body} of the
rule, and $\asp{w,p, t_1\ldots t_k}$ are all terms with $\asp{w}$ called the
weight and $\asp{p}$ the priority level.  We will refer to
$\asp{[w@p,t_1,\ldots,t_k]}$ as the \emph{tail} of the weak constraint.  A
ground instance of a weak constraint $W$ is obtained by replacing all variables
in $W$ (including those in the tail of $W$) with ground terms.  In this paper,
it is assumed that all weights and levels of all ground instances of weak
constraints are integers.
 
Given a program $P$ and an interpretation $I$ we can construct the set of
tuples $\asp{(w, p, t_1,\ldots,t_k)}$ such that there is a ground instance of a
weak constraint in $P$ whose body is satisfied by $I$ and whose (ground) tail
is $\asp{[w@p, t_1, \ldots, t_k]}$.  At each level $\asp{p}$ the \emph{score}
of $I$ is
the sum of the weights of tuples with level $\asp{p}$. An
interpretation $I_1$ $\emph{dominates}$ another interpretation $I_2$ if there is
a level $\asp{p}$ for which $I_1$ has a lower score than $I_2$, and no level higher than
$\asp{p}$ for which the scores of $I_1$ and $I_2$ are unequal.  We write $I_1 \prec_{P} I_2$
to denote that given the weak constraints in $P$, $I_1$ dominates $I_2$.

\begin{example}\label{eg:weak}
  Consider the set $WS=$
  {\small
    $\left\{\begin{array}{l}
    \asp{:\sim mode(L, walk), crime\_rating(L, R), R > 3\ruleend[1@3, L, R]}\\
    \asp{:\sim mode(L, bus)\ruleend[1@2, L]}\\
    \asp{:\sim mode(L, walk), distance(L, D)\ruleend[D@1, L, D]}\\
  \end{array}\right.$
}\\

The first weak constraint in $WS$, at priority $\asp{3}$, means
``minimise the number of legs in our journey in which we have to walk through
an area with a crime rating higher than $\asp{3}$''. As this has the highest
priority, answer sets are evaluated over this weak constraint first. The
remaining weak constraints are considered only for those answer sets that have
an equal number of legs where we have to walk through an area with such a crime
rating. The second weak constraint means ``minimise the number of buses we have
to take'' (at priority $\asp{2}$). Finally, the last weak constraint means
``minimise the distance walked''. Note that this is the case because for each
leg where we have to walk, we pay the penalty of the distance of that leg (so
the total penalty at level $\asp{1}$ is the sum of the distances of the walking
legs).
\end{example}

We now briefly summarise the key properties of Learning from Ordered Answer
Sets and ILASP2, which we extend in this paper to Context-dependent Learning
from Ordered Answer Sets and ILASP2i.
%
It makes use of two types of examples: partial
interpretations and ordering examples. A partial interpretation $e$ is a pair
of sets of atoms $\langle e^{inc}, e^{exc}\rangle$. An answer set $A$
\emph{extends}
$e$ if $e^{inc} \subseteq A$ and $e^{exc} \cap A = \emptyset$.  An
ordering example is a pair of partial interpretations.
A program $P$ \emph{bravely} (resp.  \emph{cautiously})
\emph{respects} an ordering example $\langle e_1, e_2\rangle$ if for at least one
(resp. every) pair of answer sets $\langle A_1, A_2\rangle$ that extend $e_1$
and $e_2$, it is the case that $A_1 \prec_{P} A_2$.

\begin{definition}{\cite{ICLP15}}
\label{def:LOASTheory} 
A \emph{Learning from
  Ordered Answer Sets} ($ILP_{LOAS}$) task $T$ is a tuple $\langle B, S_M,
  E\rangle$ where $B$ is an ASP program, called the background knowledge,
  $S_{M}$ is the set of rules allowed in hypotheses (the hypothesis
  space) and $E$ is a tuple $\langle E^{+}, E^{-}, O^{b}, O^{c}\rangle$.
  $E^{+}$ and $E^{-}$ are finite sets of partial interpretations called,
  respectively, positive and negative examples. $O^{b}$ and $O^{c}$ are
  finite sets of ordering examples over $E^{+}$ called, respectively, brave and
  cautious orderings. A hypothesis $H$ is an inductive solution of $T$ (written
  $H\!\in\!ILP_{LOAS}(T)$) iff:
%
      $H \subseteq S_M$;
      $\forall e \in E^{+}$,  $\exists A \in AS(B\cup H)$ st $A$ extends $e$;
      $\forall e \in E^{-}$, $\nexists A \in AS(B\cup H)$ st $A$ extends
      $e$;
      $\forall o \in O^b$, $B\cup H$ bravely respects $o$; and,
      $\forall o \in O^c$, $B\cup H$ cautiously respects $o$.
\end{definition}

In~\cite{ICLP15}, we proposed a learning algorithm, called ILASP2, and proved
that it is sound and complete with respect to $ILP_{LOAS}$ tasks.  We use the
notation $ILASP2(\langle B, S_M, E\rangle)$ to denote a function that uses
ILASP2 to return an optimal (shortest in terms of number of literals) solution
of the task $\langle B, S_M, E\rangle$.  ILASP2 terminates for any task such
that $B \cup S_M$ grounds finitely (or equivalently, $\forall H \subseteq S_M$,
$B \cup H$ grounds finitely). We call any such task \emph{well defined}.

\section{Context-dependent Learning from Ordered Answer Sets}
\label{sec:approach}
In this section, we present an extension to the $ILP_{LOAS}$ framework called
\emph{Context-dependent Learning from Ordered Answer Sets} (written
$ILP_{LOAS}^{context}$). In this new learning framework, examples can be given
with an extra background knowledge called the \emph{context} of an example. The
idea is that each context only applies to a particular example, giving more
structure to the background knowledge.

\begin{definition}\label{def:context}
  A \emph{context-dependent partial interpretation} (CDPI) is a pair $\langle
  e, C\rangle$, where $e$ is a partial interpretation and $C$ is an ASP program
  with no weak constraints, called a \emph{context}.
  A \emph{context-dependent ordering example} (CDOE) $o$ is a pair of CDPIs,
  $\langle \langle e_1, C_1\rangle, \langle e_2, C_2\rangle\rangle$. A program
  $P$ is said to \emph{bravely} (resp. \emph{cautiously}) \emph{respect} $o$ if
  for at least one (resp. every) pair $\langle A_1, A_2\rangle$ such that $A_1
  \in AS(P \cup C_1)$, $A_2 \in AS(P \cup C_2)$, $A_1$ extends $e_1$ and $A_2$
  extends $e_2$, it is the case that $A_1 \prec_{P} A_2$.
\end{definition}

\begin{example}{Consider the programs $P  =  \left\{\asp{ coin(1\ruleend\ruleend 2)\ruleend\; 1
\lbrace val(C, h), val(C, t)\rbrace 1 \codeif\! coin(C)\ruleend}\right\}$,
  $C_1 = \left\{ \asp{val(1, V) \codeif val(2,
  V)\ruleend}\right\}$ and $C_2 =\left\{ \asp{\codeif val(1, V), val(2,
V)\ruleend} \right\}$.}
$AS(P \cup C_1) = \left\{ \lbrace \asp{val(1, h),}\right.$ $\asp{val(2,
h)}\rbrace,$ $\lbrace \asp{val(1, t),}$ $\left.\asp{val(2, t)}\rbrace \right\}$
and $AS(P \cup C_2) = \left\{ \lbrace \asp{val(1, h),}\right.$\break
  $\left.\asp{val(2, t)}\rbrace, \lbrace \asp{val(1, t),val(2, h)}\rbrace
\right\}$.  Also consider the CDOE $o = \langle \langle e_1, C_1\rangle,\langle
e_2, C_2\rangle\rangle$, where $e_1 = e_2 = \langle \emptyset, \emptyset\rangle$,
Let $W=\left\{\asp{ :\sim val(C, t)\ruleend[1@1, C]}\right\}$. $P \cup W$
bravely respects $o$ as $\lbrace \asp{val(1, h),val(2, h)}\rbrace$ is preferred
to $\lbrace \asp{val(1, h),val(2, t)}\rbrace$, but does not cautiously
respect $o$ as $\lbrace \asp{val(1, t),val(2, t)}\rbrace$ is not preferred to
$\lbrace \asp{val(1, h),val(2, t)}\rbrace$.  \end{example}

Examples with empty contexts are equivalent to examples in
$ILP_{LOAS}$. Note that contexts do not contain weak constraints. The operator
$\prec_{P}$ defines the ordering over two answer sets based on the
weak constraints in one program $P$. So, given a CDOE $\langle \langle e_1,
C_1\rangle, \langle e_2, C_2\rangle\rangle$, in which $C_1$ and $C_2$ contain
different weak constraints, it is not clear
whether the ordering should be checked using the weak constraints in $P$, $P
\cup C_1$, $P \cup C_2$ or $P \cup C_1 \cup C_2$. 
We now present 
the $ILP_{LOAS}^{context}$
framework. 

\begin{definition} \label{def:loas_context} A \emph{Context-dependent Learning
  from Ordered Answer Sets} ($ILP_{LOAS}^{context}$) task is a tuple $T=\langle
  B, S_M, E\rangle$ where $B$ is an ASP program called the background
  knowledge, $S_{M}$ is the set of rules allowed in the hypotheses (the
  hypothesis space) and $E$ is a tuple $\langle E^{+}, E^{-}, O^{b},
  O^{c}\rangle$ called the examples. $E^{+}$ and $E^{-}$ are
  finite sets of CDPIs called, respectively, positive and negative examples,
  and $O^{b}$ and $O^{c}$ are finite sets of CDOEs over $E^{+}$ called,
  respectively, brave and cautious orderings.  A hypothesis $H$ is an inductive
  solution of $T$ (written $H \in ILP_{LOAS}^{context}(T)$) if and only if:

  \begin{enumerate}
    \item
      $H \subseteq S_M$;
    \item
      $\forall \langle e, C\rangle \in E^{+}$,  $\exists A \in AS(B\cup C \cup H)$ st $A$ extends $e$;
    \item
      $\forall \langle e, C\rangle \in E^{-}$, $\nexists A \in AS(B\cup C \cup H)$ st $A$ extends $e$;
    \item
      $\forall o \in O^b$, $B\cup H$ bravely respects $o$; and finally,
    \item
      $\forall o \in O^c$, $B\cup H$ cautiously respects $o$.
  \end{enumerate}
\end{definition}


\noindent
In this paper we will say a hypothesis \emph{covers} an example iff
it satisfies the appropriate condition in (2)-(5); e.g.\ a brave CDOE is
covered iff it is bravely respected.

\begin{example}
  In general, it is not the case that an $ILP_{LOAS}^{context}$ task can be
  translated into an $ILP_{LOAS}$ task simply by moving all the contexts into
  the background knowledge ($B \cup C_1 \cup \ldots \cup C_n$ where
  $C_1,\ldots,C_n$ are the contexts of the examples). Consider, for instance,
  the $ILP_{LOAS}^{context}$ task $\langle B, S_M, \langle E^{+}, E^{-}, O^{b},
  O^{c}\rangle\rangle$ defined as follows:

  \begin{itemize}
  \item $B = \emptyset$.
    $E^{-} = \emptyset$. $O^{b} = \emptyset$. $O^{c} = \emptyset$
  \item $S_M = \lbrace \asp{go\_out\codeif raining\ruleend}\;\;\;
  \asp{go\_out\codeif \naf raining\ruleend}\rbrace$
  \item $E^{+} = \left\{
      \langle\langle \lbrace \asp{go\_out}\rbrace, \emptyset\rangle, \emptyset\rangle,\;\;
      \langle\langle \emptyset, \lbrace \asp{go\_out}\rbrace\rangle, \lbrace \asp{raining\ruleend}\rbrace\rangle
    \right\}$
  \end{itemize}

\noindent
  This task has one solution: $\asp{go\_out\codeif \naf raining\ruleend}$
  But, if we were to add all the contexts to the background knowledge, we would
  get a background knowledge containing the single fact $\asp{raining}$. So,
  there would be no way of explaining both examples, as every hypothesis would,
  in this case,  lead to a single answer set (either $\lbrace \asp{raining},
  \asp{go\_out}\rbrace$ or $\lbrace \asp{raining}\rbrace$), and therefore cover
  only one of the examples.

  To capture, instead, the meaning of context-dependent examples accurately, we
  could augment the background knowledge with the choice rule $\asp{0\lbrace
  raining \rbrace} 1$ and define the $ILP_{LOAS}$ examples as the pairs
%
      $\langle \lbrace \asp{go\_out}\rbrace, \lbrace \asp{raining}\rbrace\rangle$ and 
      $\langle \lbrace \asp{raining}\rbrace,$\break $\lbrace \asp{go\_out}\rbrace\rangle$. 
%
In this way, answer sets of the inductive solution would exclude
$\asp{go\_out}$ when $\asp{raining}$ (i.e., in the context of raining), and
include $\asp{go\_out}$ otherwise, which is the correct meaning of the given
context-dependent examples.
\end{example}

Definition~\ref{def:translation} gives a general translation of
$ILP_{LOAS}^{context}$ to $ILP_{LOAS}$, which enables the use of ILASP2 to
solve $ILP_{LOAS}^{context}$ tasks. The translation assumes that each
example $ex$ has a unique (constant) identifier, $\asp{ex_{id}}$, and that for
any CDPI $ex = \langle\langle e^{inc}, e^{exc}\rangle, C\rangle$, $c(ex)$ is
the partial interpretation $\langle e^{inc} \cup \lbrace
\asp{ctx(ex_{id})}\rbrace, e^{exc}\rangle$, where $\asp{ctx}$ is a new
predicate.
Also, for any program $P$ and any atom $\asp{a}$, $append(P, a)$ is the
program constructed by appending $\asp{a}$ to the body of every rule in $P$.

\begin{definition}\label{def:translation}
  For any $ILP_{LOAS}^{context}$ task $T = \langle B_{1}, S_{M},
  \langle E^{+}_{1}, E^{-}_{1}, O^{b}_{1}, O^{c}_{1}\rangle\rangle$,
  $\mathcal{T}_{LOAS}(T) = \langle B_{2}, S_{M},$\break $\langle E^{+}_{2}, E^{-}_{2},
  O^{b}_{2}, O^{c}_{2}\rangle\rangle$, where the components of $\mathcal{T}_{LOAS}(T)$
  are as follows:

  \begin{itemize} \item
      $B_{2} = B_{1} \cup \lbrace append(C, \asp{ctx(ex_{id})}) \mid ex =
      \langle e, C\rangle \in E^{+}_{1} \cup E^{-}_{1}\rbrace$\\
      $\mbox{\hspace{6mm}} \cup \left\{\asp{1 \lbrace ctx(id_{1}), \ldots,
          ctx(id_{n})\rbrace 1\ruleend}\middle| \lbrace id_{1}, \ldots,
          id_{n}\rbrace = \lbrace ex_{id} \mid ex \in E^{+}_{1} \cup
        E^{-}_{1}\rbrace\right\}$
    \item
      $E^{+}_{2} = \lbrace c(ex) \mid ex \in E^{+}_{1}\rbrace$;
      $E^{-}_{2} = \lbrace c(ex) \mid ex \in E^{-}_{1}\rbrace$
    \item
      ${O^{b}_{2}\!=\!\lbrace\langle c(ex_1),c(ex_2)\rangle\!\mid\!\langle
        ex_1, ex_2\rangle\!\in\!O^{b}_{1}\rbrace; O^{c}_{2}\!=\!\lbrace
        \langle c(ex_1),c(ex_2)\rangle\!\mid\!\langle ex_1,
      ex_2\rangle\!\in\!O^{c}_{1}\rbrace}$
 \end{itemize}
\end{definition}

We say that an $ILP_{LOAS}^{context}$ task $T$ is well defined if and only if
$\mathcal{T}_{LOAS}(T)$ is a well defined $ILP_{LOAS}$ task. Before proving
that this translation is correct, it is useful to introduce a lemma (which is
proven in \ref{sec:proofs}). Given a program $P$ and a set of contexts
$C_1,\ldots, C_n$, Lemma~\ref{lem:context} gives a way of combining the alternative
contexts into the same program. Each rule of each context $C_i$, is appended
with a new atom $\asp{a_i}$, unique to $C_i$, and a choice rule stating that
exactly one of the new $\asp{a_i}$ atoms is true in each answer set. This means
that the answer sets of $P\cup C_i$, for each $C_i$, are the answer sets of the
combined program that contain $\asp{a_i}$ (with the extra atom $\asp{a_i}$).

\begin{lemma}\label{lem:context}

  For any program $P$ (consisting of normal rules, choice rules and constraints)
  and any set of pairs $S = \lbrace \langle C_1, \asp{a_1}\rangle, \ldots,
  \langle C_n, \asp{a_n}\rangle\rbrace$ such that none of the atoms $\asp{a_i}$
  appear in $P$ (or in any of the $C$'s) and each $\asp{a_i}$ atom is unique:
  $AS(P \cup \left\{ \asp{1 \lbrace a_1, \ldots, a_n \rbrace 1\ruleend}\right\}
  \cup \left\{ append(C_i, \asp{a_i}) \middle| \langle C_i, \asp{a_i}\rangle \in
  S\right\})= \left\{ A \cup \lbrace \asp{a_i}\rbrace \middle| A \in AS(P \cup
  C_i), \langle C_i, \asp{a_i}\rangle \in S\right\}$
\end{lemma}

\begin{theorem}\label{thm:translation}
  For any $ILP_{LOAS}^{context}$ learning task $T$,
  $ILP_{LOAS}(\mathcal{T}_{LOAS}(T)) = ILP_{LOAS}^{context}(T)$.
\end{theorem}

\begin{proof}

  Let $T = \langle B_{1}, S_M, \langle E^{+}_{1}, E^{-}_{1}, O^{b}_{1},
  O^{c}_{1}\rangle\rangle$ and $\mathcal{T}_{LOAS}(T) = \langle B_{2},
  S_M,\langle E^{+}_{2}, E^{-}_{2},O^{b}_{2}, O^{c}_{2}\rangle\rangle$.

  \leftskip=4mm

  \noindent $\mbox{\hspace{-5.5mm}}$
  $H \in ILP_{LOAS}^{context}(T)$
  $\Leftrightarrow H \subseteq S_M;$ $\forall \langle e, C\rangle \in E^{+}_{1}, \exists A \in AS(B_{1} \cup C \cup H)$ st $A$ extends $e$;
                                     $\forall \langle e, C\rangle \in E^{-}_{1}, \nexists A \in AS(B_{1} \cup C \cup H)$ st $A$ extends $e$;
                                     $\forall o \in O^{b}_{1}, B_{1} \cup H$ bravely respects $o$;
                                     $\forall o \in O^{c}_{1}, B_{1} \cup H$ cautiously respects $o$

  \noindent $\mbox{\hspace{-5.5mm}}$
  $\Leftrightarrow H \subseteq S_M;$ $\forall ex \in E^{+}_{1}, \exists A \in AS(B_{2} \cup H)$ st $A$ extends $c(ex)$;
                                     $\forall ex \in E^{-}_{1}, \nexists A \in AS(B_{2} \cup H)$ st $A$ extends $c(ex)$;
                                     $\forall \langle ex_1, ex_2 \rangle \in O^{b}, B_{2} \cup H$ bravely respects $\langle c(ex_1),$\break $c(ex_2)\rangle$;
                                     $\forall \langle ex_1, ex_2 \rangle \in O^{c}, B_{2} \cup H$ cautiously respects $\langle c(ex_1), c(ex_2)\rangle$ (by Lemma~\ref{lem:context})

  \noindent $\mbox{\hspace{-5.5mm}}$
  $\Leftrightarrow H \subseteq S_M;$ $\forall e \in E^{+}_{2}, \exists A \in AS(B_{2} \cup H)$ st $A$ extends $e$;
                                     $\forall e \in E^{-}_{2}, \nexists A \in AS(B_{2} \cup H)$ st $A$ extends $e$;
                                     $\forall o\in O^{b}, B_{2} \cup H$ bravely respects $o$;
                                     $\forall o\in O^{c}, B_{2} \cup H$ cautiously respects $o$
%
\end{proof}

\noindent
Theorem~\ref{thm:translation} shows that, by using an automatic
$\mathcal{T}_{LOAS}$ translation,
ILASP2 can be used to solve $ILP_{LOAS}^{context}$ tasks.  Although this means
that any $ILP_{LOAS}^{context}$ task can be translated to an $ILP_{LOAS}$ task,
context-dependent examples are useful for two reasons: firstly, they simplify
the representation of some learning tasks; and secondly, the added structure
gives more information about which parts of the background knowledge apply to
particular examples.  In Section~\ref{sec:algorithm} we present a new algorithm
that is able to take advantage of this extra information.

\begin{theorem}\label{thm:complexity}
  The complexity of deciding whether an $ILP_{LOAS}^{context}$ task is
  satisfiable is $\Sigma^{P}_{2}$-complete.
\end{theorem}

\noindent
Theorem~\ref{thm:complexity} (proven in \ref{sec:proofs}) implies that the
complexity of deciding the satisfiability of an $ILP_{LOAS}^{context}$ task is
the same as for an $ILP_{LOAS}$ task. Note that, similar to Theorem 2
in~\cite{ICLP15}, this result is for propositional tasks.

\section{Iterative Algorithm: ILASP2i}
\label{sec:algorithm}
In the previous section, we showed that our new $ILP_{LOAS}^{context}$ task can
be translated into $ILP_{LOAS}$, and therefore solved using the ILASP2
algorithm~\cite{ICLP15}.  However, ILASP2 may suffer from
scalability issues, due to the number of examples or the size and complexity of
the grounding of the hypothesis space, when combined with the background
knowledge.  In this paper, we address the first scalability issue by
introducing a new algorithm, ILASP2i, for solving (context-dependent) learning
from ordered answer sets tasks. The algorithm {\em iteratively} computes a
hypothesis by incrementally constructing a subset of the examples
that are {\em relevant} to the search. These are essentially counterexamples
for incorrect hypotheses. The idea of the algorithm is to incrementally build,
during the computation, a set of relevant examples and, at each iterative step,
to learn hypotheses with respect only to this set of relevant examples instead
of the full set of given examples.  Although we do not directly address the
second issue of large and complicated hypothesis spaces, it is worth noting
that by using the notion of context-dependent examples, the size of the
background knowledge (and therefore the grounding of the hypothesis space) in a
particular iteration of our algorithm may be much smaller. In fact, in
Section~\ref{sec:benchmarks} we show that the background knowledge of one
learning task (learning the definition of a Hamiltonian graph) can be
eliminated altogether by using contexts.

%
\begin{definition}\label{def:relevant}
  Consider an $ILP_{LOAS}^{context}$ learning task $T = \langle B, S_M, \langle
  E^{+}, E^{-}, O^{b}, O^{c}\rangle\rangle$ and a hypothesis $H \subseteq S_M$.
  A (context-dependent) example $ex$ is \emph{relevant} to $H$ given $T$ if $ex
  \in E^{+}\cup E^{-}\cup O^{b}\cup O^{c}$ and $B \cup H$ does not cover $ex$.
\end{definition}

The intuition of ILASP2i (Algorithm~\ref{alg:ILASP2i}) is that we start with an
empty set of \emph{relevant} examples and an empty hypothesis. At each step of
the search we look for an example which is relevant to our current hypothesis
(i.e. an example that $B \cup H$ does not cover).  If no such example exists,
then we return our current hypothesis as an optimal inductive solution;
otherwise, we add the example to our relevant set of examples and use ILASP2 to
compute a new hypothesis. 

The notation $<<$, in line 5 of algorithm~\ref{alg:ILASP2i}, means to add the
relevant example $re$ to the correct set in $Relevant$ (the first set if it is
a positive example etc).

\begin{algorithm}
  \begin{algorithmic}[1]
    \Procedure{\algname}{$\langle B, S_M, E \rangle$}
      \State{$Relevant = \langle \emptyset, \emptyset, \emptyset, \emptyset\rangle;\;\; H = \emptyset;$}
      \State{$re = findRelevantExample(\langle B, S_M, E \rangle, H);$}
      \While{$re \neq \mathtt{nil}$}
        \State{$Relevant << re;$}
        \State{$H = ILASP2(\mathcal{T}_{LOAS}(\langle B, S_M, Relevant\rangle));$}
        \State{\textbf{if}($H == \asp{nil}$)$\;\;$\textbf{return}$\;\;\asp{UNSATISFIABLE};$}
        \State{\textbf{else}$\;\;re = findRelevantExample(\langle B, S_M, E\rangle, H);$}
      \EndWhile
      \State{\textbf{return}\,$H;$}
    \EndProcedure
  \end{algorithmic}
  \caption{\algname \label{alg:ILASP2i}}
\end{algorithm}

The function $findRelevantExample(\langle B, S_M, E\rangle, H)$ returns a
(context-dependent) example in $E$ which is not covered by $B \cup H$, or
$\mathtt{nil}$ if no such example exists.  It works by encoding $B \cup H$ and
$E$ into a meta program whose answer sets can be used to determine which
examples in $E$ are covered. This meta program contains a choice rule, which
specifies that each answer set of the program tests the coverage of a single
CDPI or CDOE example. For a positive or negative example $ex = \langle e,
C\rangle$, if there is an answer set of the meta program corresponding to $ex$
then there must be at least one answer set of $B \cup C \cup H$ that extends
$e$. This means that positive (resp. negative) examples are covered iff there
is at least one (resp. no) answer set of the meta program that corresponds to
$ex$.  Similarly, CDOE's are encoded such that each brave (resp. cautious)
ordering $o$ is respected iff there is at least one (resp. no) answer set
corresponding to $o$.  $findRelevantExamples$ uses the answer sets of the
meta program to determine which examples are not covered.  Details of the meta
program are in~\ref{sec:meta}, including proof of its correctness.

It should be noted that in the worst case our set of relevant examples is
equal to the entire set of examples. In this case, \algname is
slower than ILASP2. In real settings, however, as examples are not carefully
constructed, there is likely to be overlap between examples, so the
relevant set will be much smaller than the whole set.\break
Theorem~\ref{thm:terminate} shows that \algname has the same condition for
termination as ILASP2.

\begin{theorem}\label{thm:terminate}
  \algname terminates for any well defined $ILP_{LOAS}^{context}$ task.
\end{theorem}

\noindent
Note that although the algorithm is sound, it is complete only in the sense
that it always returns an optimal solution if one exists (rather than returning
the full set).

\begin{theorem}\label{thm:sac}
  \algname is sound for any well defined $ILP_{LOAS}^{context}$ task, and
  returns an optimal solution if one exists.
\end{theorem}

\noindent
Note that in Algorithm~\ref{alg:ILASP2i} the translation of a context-dependent
learning task is applied to the context-dependent task generated incrementally
at each step of the iteration (see line 6) instead of pre-translating the full
initial task. This has the advantage that the background knowledge of the
translated task only contains the contexts for the relevant examples, rather
than the full set.  In Section~\ref{sec:benchmarks} we compare the efficiency
of ILASP2i on $ILP_{LOAS}^{context}$ tasks that have been pre-translated with
corresponding tasks that have not been pre-translated, and demonstrate that in
the latter case ILASP2i can be up to one order of magnitude faster. We refer to
the application\break of ILASP2i with an automatic pre-translation to $ILP_{LOAS}$
as ILASP2i\_pt.

\section{Evaluation}
\label{sec:benchmarks}

In this section, we demonstrate the improvement in performance of ILASP2i over
ILASP2, both in terms of running time and memory usage.
Although there are benchmarks for ASP solvers, as ILP systems for ASP are
relatively new, and solve different computational tasks, there are no
benchmarks for learning ASP programs. We therefore investigate new problems.
To demonstrate the increased performance of ILASP2i over ILASP2, we chose tasks
with large numbers of examples.
We compare the algorithms in four problem settings, each including tasks
requiring different components of the $ILP_{LOAS}^{context}$ framework.  We
also investigate how the performance and accuracy vary with the number of
examples, for the task of learning user journey preferences.
All learning tasks were run with ILASP2, ILASP2i and ILASP2i\_pt\footnote{For
details of the tasks discussed in this section and how to download and run
ILASP2, ILASP2i and ILASP2i\_pt, see
\url{https://www.doc.ic.ac.uk/~ml1909/ILASP}.}.


\begin{table}[h]
    \resizebox{\textwidth}{!}{%
  \begin{tabular}{p{0.16\textwidth}p{0.025\textwidth}p{0.025\textwidth}p{0.025\textwidth}
      p{0.04\textwidth}
      p{0.05\textwidth}
      p{0.05\textwidth}
      p{0.055\textwidth}
      p{0.07\textwidth}
      p{0.07\textwidth}
    p{0.07\textwidth}}
  \hline
  \hline
  Learning & \multicolumn{4}{c}{\#examples} &\multicolumn{3}{c}{time/s} & \multicolumn{3}{c}{Memory/kB}\\
  task & $E^{+}$ & $E^{-}$ & $O^{b}$ & $O^{c}$ & 2 & 2i\_pt & 2i & 2 & 2i\_pt & 2i\\\hline

  Hamilton A                     & 100 & 100 & 0   & 0   & 10.3 & 4.2 & 4.3 & 9.7$\times 10^{4}$ & 1.2$\times 10^{4}$ & 1.2$\times 10^{4}$ \\
  Hamilton B                     & 100 & 100 & 0   & 0   & 32.0 & 84.9 & 3.6 & 3.6$\times 10^{5}$ & 2.7$\times 10^{5}$ & 1.4$\times 10^{4}$ \\
  Scheduling A                   & 400 & 0   & 110 & 90  & 291.9 & 64.2 & 63.4 & 2.7$\times 10^{6}$ & 1.7$\times 10^{5}$ & 1.7$\times 10^{5}$ \\
  Scheduling B                   & 400 & 0   & 128 & 72  & 347.2 & 40.1 & 40.3 & 5.2$\times 10^{6}$ & 2.6$\times 10^{5}$ & 2.6$\times 10^{5}$ \\
  Scheduling C                   & 400 & 0   & 133  & 67 & 1141.8 & 123.6 & 124.2 & 8.4$\times 10^{6}$ & 4.9$\times 10^{5}$ & 5.0$\times 10^{5}$ \\
  Agent A                        & 200 & 0   & 0   & 0   & 444.5 & 56.7 & 39.1 & 4.7$\times 10^{6}$ & 3.7$\times 10^{5}$ & 9.8$\times 10^{4}$ \\
  Agent B                        & 50  & 0   & 0   & 0   & TO & 212.3 & 9.4 & TO & 1.1$\times 10^{6}$ & 1.8$\times 10^{5}$ \\
  Agent C                        & 80  & 120 & 0   & 0   & 808.7 & 132.3 & 60.1 & 2.9$\times 10^{6}$ & 3.5$\times 10^{5}$ & 8.4$\times 10^{4}$ \\
  Agent D                        & 172 & 228 & 390 & 0   & OOM & 863.3 & 408.4 & OOM & 2.4$\times 10^{6}$ & 8.0$\times 10^{5}$ \\\hline\hline

\end{tabular}
}

\caption{The running times of ILASP2, ILASP2i and ILASP2i\_pt. TO stands for
  time out (6 hours) and OOM stands for out of memory.\label{tbl:benchmarks}}

\end{table}

Our first problem setting is learning the definition of whether a graph is
Hamiltonian or not (i.e. whether it contains a Hamilton cycle). Hamilton A is
an $ILP_{LOAS}$ (non context-dependent) task. The background knowledge $B$
consists of the two choice rules $\asp{1}$ $\asp{\lbrace}$ $\asp{node(1),}$
$\asp{node(2),}$ $\asp{node(3),}$ $\asp{node(4)}$ $\asp{\rbrace 4}\,$ and
$\,\asp{0}$ $\asp{\lbrace}$ $\asp{edge(N1, N2)}$ $\asp{\rbrace}$ $\asp{1}$
$\asp{\codeif}$\break $\asp{node(N1), node(N2)}$, meaning that the answer sets
of $B$ correspond to
the graphs of size 1 to 4. Each example then corresponds to exactly one graph,
by specifying which $\asp{node}$ and $\asp{edge}$ atoms should be true.
Positive examples correspond to Hamiltonian graphs, and negative examples
correspond to non-Hamiltonian graphs. Hamilton B is an $ILP_{LOAS}^{context}$
encoding of the same problem. The background knowledge is empty, and each
example has a context consisting of the $\asp{node}$ and $\asp{edge}$ atoms
representing a single graph. ILASP2i performs significantly better than ILASP2
in both cases. Although ILASP2i is slightly faster at solving Hamilton B
compared with Hamilton A, one interesting result is that ILASP2 and
ILASP2i\_pt perform better on Hamilton A. This is because the non
context-dependent encoding in Hamilton A is more efficient than the
automatic translation (using definition~\ref{def:translation}) of Hamilton B.

To test how the size of the contexts affects the performance of the three
algorithms, we reran the Hamilton A and B experiments with the maximum size of
the graphs varying from 4 to 10. Each experiment was run 100 times with
randomly generated sets of positive and negative examples (100 of each in each
experiment). The results (figure~\ref{fig:hamilton}) show that ILASP2i performs
best in both cases - interestingly, on average, there is no difference between
Hamilton A (non context-dependent) and Hamilton B (context-dependent) at first,
but as the maximum graph size increases, the domain of the background knowledge
in Hamilton A increases and so ILASP2i performs better on Hamilton B.
Although ILASP2i\_pt is much slower on Hamilton B than Hamilton A, it uses
significantly less memory on the former. As the performance of ILASP2i and
ILASP2i\_pt is the same on any non context-dependent task, we do not show the
results for ILASP2i\_pt on Hamilton A.

\begin{figure*}
  \begin{multicols}{2}
    \includegraphics[width=0.7\textwidth,height=0.38\textwidth, trim={80mm 1mm 20mm 6mm},clip]{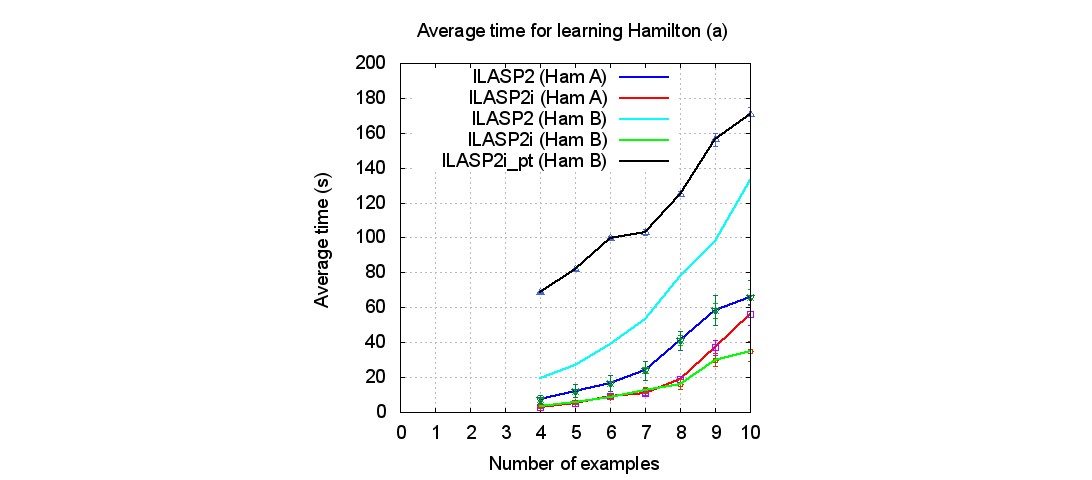}

    \includegraphics[width=0.7\textwidth,height=0.38\textwidth, trim={75mm 1mm 20mm 6mm},clip]{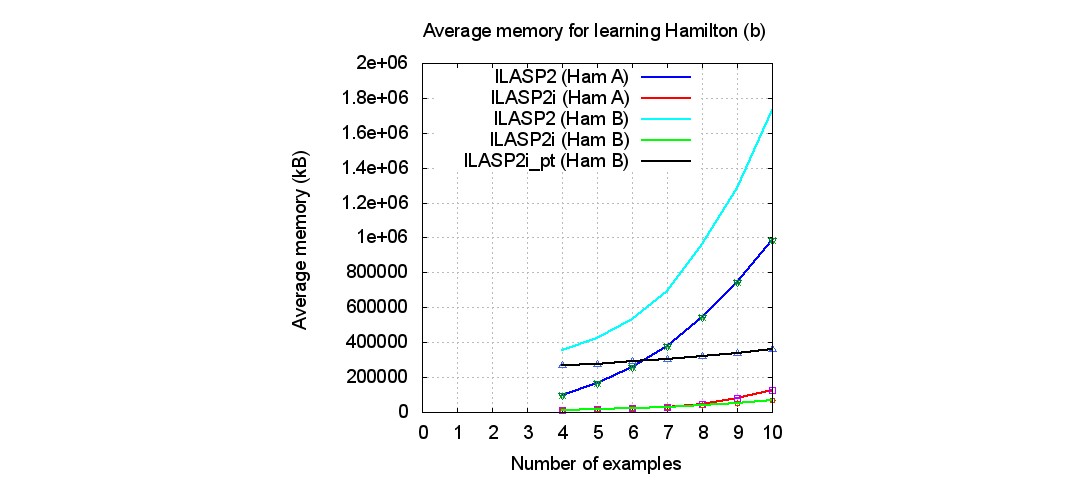}

  \end{multicols}

    \caption{\label{fig:hamilton} (a) the average computation time and (b) the memory usage of ILASP2, ILASP2i and ILASP2i\_pt for Hamilton A and B.}

\end{figure*}

We also reconsider the problem of learning scheduling preferences, first
presented in~\cite{ICLP15}. In this setting, the goal is to learn an academic's
preferences about interview scheduling, encoded as weak constraints. Tasks A-C
in this case are over examples with \texttt{3x3}, \texttt{4x3} and \texttt{5x3}
timetables, respectively. As this setting contains no contexts for the
examples, the performance of ILASP2i and ILASP2i\_pt are relatively similar;
however, for larger timetables both are over an order of magnitude faster and
use over an order of magnitude less memory than ILASP2. Interestingly, although
ILASP2i does not directly attempt to scale up the size of possible problem
domains (in this case, the dimensions of the timetables), this experiment
demonstrates that ILASP2i does (indirectly) improve the performance on larger
problem domains. One unexpected observation is that ILASP2i runs faster on task
B than task A. This is caused by the algorithm choosing ``better'' relevant
examples for task B, and therefore needing a smaller set of relevant examples.
On average, the time for \texttt{4x3} timetables would be expected to be higher
than the \texttt{3x3}'s.

Our third setting is taken from~\cite{JELIA_ILASP} and is based on an agent
learning the rules of how it is allowed to move within a grid. Agent A requires
a hypothesis describing the concept of which moves are valid, given a history
of where an agent has been. Agent B requires a similar hypothesis to be
learned, but with the added complexity that an additional concept is required
to be invented. While Agent A and Agent B are similar to scenarios 1 and 2
in~\cite{JELIA_ILASP}, the key difference is that different examples contain
different histories of where the agent has been. These histories are encoded as
contexts, whereas in~\cite{JELIA_ILASP}, one single history was encoded in the
background knowledge. There are also many more examples in these experiments.
In Agent C, the hypothesis from Agent A must be learned along with a constraint
ruling out histories in which the agent visits a cell twice (not changing the
definition of valid move). This requires negative examples to be given, in
addition to positive examples. In Agent D, weak constraints must be learned to
explain why some traces through the grid are preferred to others. This uses
positive, negative and brave ordering examples. In each case, ILASP2i performs
significantly better than ILASP2i\_pt, which performs significantly better
than\break ILASP2 (ILASP2 times out in one experiment, and runs out of memory
in another).

In our final setting, we investigate the problem of learning a user's
preferences over alternative journeys, in order to demonstrate how the
performance of the three algorithms varies with the number of examples. We also
investigate how the accuracy of ILASP2i varies with the number of examples. In
this scenario, a user makes requests to a journey planner to get from one
location to another. The user then chooses a journey from the 
alternatives returned by the planner. A journey consists of one or more
legs, in each of which the user uses a single mode of transport.

We used a simulation environment \cite{poxrucker2014} to generate realistic
examples of journeys.  In our experiment, we ran the simulator for one
(simulated) day to generate a set of journey requests, along with the
attributes of each possible journey.  The attributes provided by the simulation
data are: $\asp{mode}$, which takes the value $\asp{bus}$, $\asp{car}$,
$\asp{walk}$ or $\asp{bicycle}$; $\asp{distance}$, which takes an integer value
between $\asp{1}$ and $\asp{20000}$; and $\asp{crime\_rating}$.  As the crime
ratings were not readily available from the simulator, we used a randomly
generated value between $\asp{1}$ and $\asp{5}$.

\begin{figure*}

    \noindent
    \includegraphics[width=1\textwidth,height=0.38\textwidth, trim={25mm 1mm 0mm 5mm},clip]{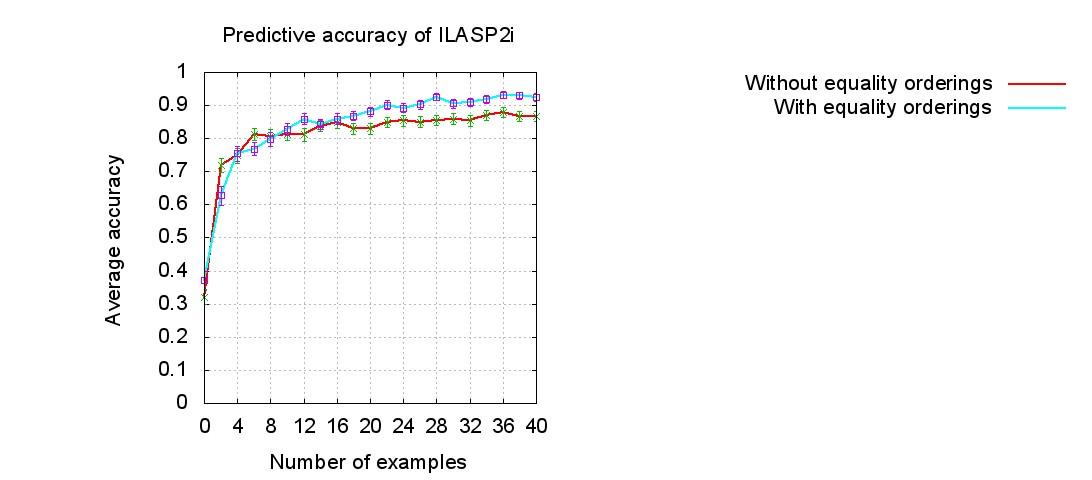}
    \caption{\label{fig:accuracy} average accuracy of ILASP2i}
\end{figure*}

For our experiments, we assume that the user's preferences can be represented
by a set of weak constraints based on the attributes of a leg. We constructed a
set of possible weak constraints, each including at most 3 literals. Most of
these literals capture the leg's attributes, e.g., $\asp{mode(L, bus)}$ or
$\asp{crime\_rating(L, R)}$ (if the attribute's values range over integers this
is represented by a variable, otherwise each possible value is used as a
constant). For the crime rating ($\asp{crime\_rating(L,R)}$), we also allow
comparisons of the form $R > \asp{c}$ where $\asp{c}$ is an integer from 1 to
4. The weight of each weak constraint is a variable representing the distance of
the leg in the rule, or 1 and the priority is 1, 2 or 3. One possible set of
preferences is the set of weak constraints in Example~\ref{eg:weak}.
$S_J$ denotes the set of possible weak constraints.

We now describe how to represent the journey preferences scenario in\break
$ILP_{LOAS}^{context}$.  We assume a journey is encoded as a set of attributes
of the legs of the journey; for example the journey $\lbrace
\asp{distance(leg(1), 2000),}$ $\asp{ distance(leg(2), 100),}$\break $\asp{mode(leg(1), bus),
mode(leg(2), walk)}\rbrace$ has two legs; in the first leg, the person must
take a bus for 2000m and in the second, he/she must walk 100m. Given a
set of such journeys $J = \lbrace j_1,\ldots,j_n\rbrace$ and a partial ordering
$O$ over $J$,
  $\mathcal{M}(J, O, S_J)$ is the $ILP_{LOAS}^{context}$ task $\langle \emptyset, S_J,
  E^{+}, \emptyset, O^{b}, \emptyset\rangle$, where
  $E^{+}=\lbrace \langle\langle\emptyset,\emptyset\rangle, j_i\rangle \mid j_i
  \in J\rbrace$ and $O^{b}=$\break $\lbrace \langle
  \langle\langle\emptyset,\emptyset\rangle, j_1\rangle,
  \langle\langle\emptyset,\emptyset\rangle, j_2\rangle\rangle \mid \langle j_1,
  j_2 \rangle \in O\rbrace$.
Each solution of $\mathcal{M}(J, O, S_J)$ is a set of weak constraints
representing preferences which explain the ordering of the journeys.
Note that the positive examples are automatically satisfied as the (empty)
background knowledge (combined with the context) already covers them.  Also, as
the background knowledge together with each context has exactly one answer set,
the notions of brave and cautious orderings coincide; hence, we do not need
cautious ordering examples for this task. Furthermore, since we are only
learning weak constraints, and not hard constraints, the task also has no
negative examples (a negative example would correspond to an invalid journey).

In each experiment we randomly selected 100 test hypotheses, each consisting of
between 1 and 3 weak constraints from $S_J$. For each test hypothesis $H_T$, we
then used the simulated journeys to generate a set of ordering examples
$\langle j_1, j_2\rangle$ such that $j_1$ was one of the optimal journeys,
given $H$, and $j_2$ was an non-optimal alternative to $j_1$. We
then tested the algorithms on tasks with varying numbers of ordering examples
by taking a random sample of the complete set of ordering examples.

\begin{figure*}
    \includegraphics[width=1\textwidth,height=0.4\textwidth, trim={10mm 0mm 0mm 0mm},clip]{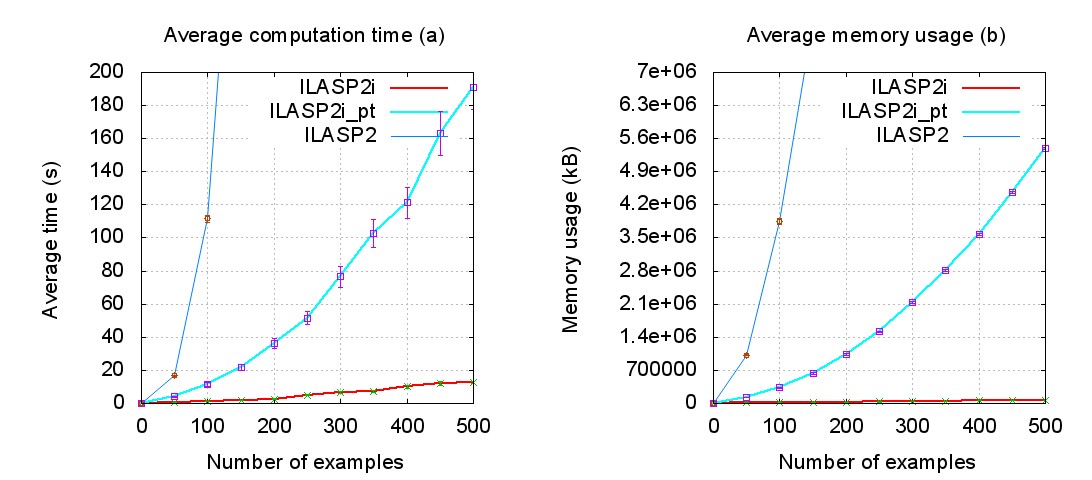}

    \caption{\label{fig:results} (a) the average computation time and (b) the memory usage of ILASP2, ILASP2i and ILASP2i\_pt for learning journey preferences.}

\end{figure*}

The accuracy of \algname for different numbers of examples is shown in
Figure~\ref{fig:accuracy}. The average accuracy converges to around $85\%$
after roughly 20 examples. As we only gave examples of journeys such that one
was preferred to the other the hypotheses were often incorrect at predicting
that two journeys were equal. We therefore introduced a new type of brave
ordering example to ILASP2i, which enables us to specify that two answer sets
should be equally optimal. We ran the same experiment with half of the ordering
examples as the new ``equality'' orderings. The average accuracy increased to
around $93\%$ after 40 examples.  Note that as ILASP2 and \algname return an
arbitrary optimal solution of a task, their accuracy results, on average, are
the same. We therefore only present the results for ILASP2i.

Figures~\ref{fig:results}(a) and (b) show the running times and memory usage
(respectively) for up to 500 examples for ILASP2, ILASP2i and ILASP2i\_pt.  For
experiments with more than 200 examples, ILASP2 ran out of memory.  By 200
examples, ILASP2i is already over 2 orders of magnitude faster and uses over 2
orders of magnitude less memory than ILASP2, showing a significant improvement
in scalability. The fact that by 500 examples ILASP2i is an order of magnitude
faster without the pre-translation shows that, in this problem domain, the
context is a large factor in this improvement; however, ILASP2i\_pt's
significantly improved performance over ILASP2 shows that the iterative nature
of ILASP2i is also a large factor.

\section{Related Work}
\label{sec:related}

Most approaches to ILP address the learning of definite
programs~\cite{srinivasan2001aleph,metagol}, usually aiming to learn
Prolog programs. As the language features of Prolog and ASP are different (e.g.
ASP lacks lists, Prolog lacks choice), a comparison is difficult. On the shared
language of ASP and the fragment of Prolog\break learned by these systems
(definite rules), a traditional ILP task can be represented with a single
positive example (where the inclusions (resp. exclusions) of this example correspond to the
positive (resp. negative) examples in the original task).

The idea of context-dependent example has similarities with the concept of
\emph{learning from interpretation transitions}
(LFIT)~\cite{inoue2014learning}, where examples are pairs of set of atoms
$\langle I, J\rangle$ such that $B \cup H$ must satisfy $T_{B \cup H}(I) =
J$ (where $T_P(I)$ is the set of immediate consequences of $I$ with respect to
the program $P$). LFIT technically learns under the supported model semantics
and uses a far smaller language than that supported by $ILP_{LOAS}^{context}$
(not supporting choice rules or hard or weak constraints), but
can be simply represented in $ILP_{LOAS}^{context}$. The head $\asp{h}$ of each
rule in the background knowledge and hypothesis space should be replaced by
$\asp{j(h)}$, and each body literal $\asp{b}$, by $\asp{i(b)}$. Each example
$\langle I, J\rangle$ should then be mapped to a context-dependent positive
example $\langle\langle \lbrace \asp{j(a)} \mid \asp{a} \in J\rbrace,\emptyset
\rangle, \lbrace \asp{i(a)\ruleend}\mid \asp{a} \in I\rbrace\rangle$.

Other than our own frameworks, the two main ILP frameworks
under the answer set semantics are \emph{brave} and \emph{cautious}
induction~\cite{Sakama2009}. As $ILP_{LOAS}^{context}$ subsumes $ILP_{LOAS}$,
$ILP_{LOAS}^{context}$ inherits the ability to perform both brave and cautious
induction. ILASP2i is therefore more general than systems such
as~\cite{ray2009nonmonotonic,Corapi2012,raspal}, which can only perform brave
induction. In ILP, learners can be divided into batch learners (those which consider all
examples simultaneously), such as~\cite{ray2009nonmonotonic,Corapi2012,raspal,JELIA_ILASP}, and learners which
consider each example in turn (using a cover loop), such
as~\cite{srinivasan2001aleph,muggleton1995inverse,ray2003hybrid}. Under the
answer set semantics, most learners are batch learners due to the
non-monotonicity. In fact, it is worth noting that, in particular, although the
HAIL~\cite{ray2003hybrid} algorithm for learning definite clauses employs a
cover loop, the later XHAIL algorithm is a batch learner as it learns
non-monotonic programs~\cite{ray2009nonmonotonic}.  One approach which did
attempt to utilise a cover loop is~\cite{sakama2005}. Their approach, however,
was only sound for a small (monotonic) fragment of ASP if the task had multiple
examples, as otherwise later examples could cause earlier examples to become
uncovered.

The ILED system~\cite{katzouris2015incremental} extended the ideas behind XHAIL
in order to allow incremental learning of event definitions.  This system takes
as input, multiple ``windows'' of examples and incrementally learns a
hypothesis.  As the approach is based on theory revision (at each step,
revising the hypothesis from the previous step), ILED is not guaranteed to
learn an optimal solution. In contrast, ILASP2i learns a new hypothesis in each
iteration and incrementally builds the set of relevant examples.

\section{Conclusion}

In this paper, we have presented an extension to our $ILP_{LOAS}$ framework
which allows examples to be given with extra background knowledge called the
context of the example. We have shown that these contexts can be used to give
structure to the background knowledge, showing which parts apply to which
examples. We have also presented a new algorithm, ILASP2i, which makes use of
this added structure to improve the efficiency over the previous ILASP2.
 In Section~\ref{sec:benchmarks}, we demonstrated that our new approach is
considerably faster for tasks with large numbers of examples.

Unlike previous systems for learning under the answer set semantics, ILASP2i is
not a batch learner and does not need to consider all examples at the same
time, but instead iteratively builds a set of relevant examples. This
combination of relevant examples and the added structure given by contexts
means that ILASP2i can be up to 2 orders of magnitude better than ILASP2, both
in terms of time and memory usage.  In future work, we intend to investigate
how to improve the scalability of \algname with larger hypothesis spaces and
with noisy examples.

\bibliographystyle{acmtrans}
\bibliography{paper}

\begin{thebibliography}{}

\bibitem[\protect\citeauthoryear{Athakravi, Corapi, Broda, and Russo}{Athakravi
  et~al\mbox{.}}{2014}]{raspal}
{\sc Athakravi, D.}, {\sc Corapi, D.}, {\sc Broda, K.}, {\sc and} {\sc Russo,
  A.} 2014.
\newblock Learning through hypothesis refinement using answer set programming.
\newblock In {\em Inductive Logic Programming}. Springer, 31--46.

\bibitem[\protect\citeauthoryear{Corapi, Russo, and Lupu}{Corapi
  et~al\mbox{.}}{2012}]{Corapi2012}
{\sc Corapi, D.}, {\sc Russo, A.}, {\sc and} {\sc Lupu, E.} 2012.
\newblock Inductive logic programming in answer set programming.
\newblock In {\em Inductive Logic Programming}. Springer, 91--97.

\bibitem[\protect\citeauthoryear{Inoue, Ribeiro, and Sakama}{Inoue
  et~al\mbox{.}}{2014}]{inoue2014learning}
{\sc Inoue, K.}, {\sc Ribeiro, T.}, {\sc and} {\sc Sakama, C.} 2014.
\newblock Learning from interpretation transition.
\newblock {\em Machine Learning\/}~{\em 94,\/}~1, 51--79.

\bibitem[\protect\citeauthoryear{Katzouris, Artikis, and Paliouras}{Katzouris
  et~al\mbox{.}}{2015}]{katzouris2015incremental}
{\sc Katzouris, N.}, {\sc Artikis, A.}, {\sc and} {\sc Paliouras, G.} 2015.
\newblock Incremental learning of event definitions with inductive logic
  programming.
\newblock {\em Machine Learning\/}~{\em 100,\/}~2-3, 555--585.

\bibitem[\protect\citeauthoryear{Law, Russo, and Broda}{Law
  et~al\mbox{.}}{2014}]{JELIA_ILASP}
{\sc Law, M.}, {\sc Russo, A.}, {\sc and} {\sc Broda, K.} 2014.
\newblock Inductive learning of answer set programs.
\newblock In {\em Logics in Artificial Intelligence (JELIA 2014)}. Springer.

\bibitem[\protect\citeauthoryear{Law, Russo, and Broda}{Law
  et~al\mbox{.}}{2015a}]{ICLP15}
{\sc Law, M.}, {\sc Russo, A.}, {\sc and} {\sc Broda, K.} 2015a.
\newblock Learning weak constraints in answer set programming.
\newblock {\em Theory and Practice of Logic Programming\/}~{\em 15,\/}~4-5,
  511--525.

\bibitem[\protect\citeauthoryear{Law, Russo, and Broda}{Law
  et~al\mbox{.}}{2015b}]{ILASP2Proof}
{\sc Law, M.}, {\sc Russo, A.}, {\sc and} {\sc Broda, K.} 2015b.
\newblock Proof of the soundness and completeness of {ILASP2}.
\newblock \url{https://www.doc.ic.ac.uk/~ml1909/Proofs_for_ILASP2.pdf}.

\bibitem[\protect\citeauthoryear{Lifschitz and Turner}{Lifschitz and
  Turner}{1994}]{splitting}
{\sc Lifschitz, V.} {\sc and} {\sc Turner, H.} 1994.
\newblock Splitting a logic program.
\newblock In {\em ICLP}. Vol.~94. 23--37.

\bibitem[\protect\citeauthoryear{Muggleton}{Muggleton}{1991}]{muggleton1991}
{\sc Muggleton, S.} 1991.
\newblock Inductive logic programming.
\newblock {\em New generation computing\/}~{\em 8,\/}~4, 295--318.

\bibitem[\protect\citeauthoryear{Muggleton}{Muggleton}{1995}]{muggleton1995inverse}
{\sc Muggleton, S.} 1995.
\newblock Inverse entailment and progol.
\newblock {\em New generation computing\/}~{\em 13,\/}~3-4, 245--286.

\bibitem[\protect\citeauthoryear{Muggleton, Lin, Pahlavi, and
  Tamaddoni-Nezhad}{Muggleton et~al\mbox{.}}{2014}]{metagol}
{\sc Muggleton, S.~H.}, {\sc Lin, D.}, {\sc Pahlavi, N.}, {\sc and} {\sc
  Tamaddoni-Nezhad, A.} 2014.
\newblock Meta-interpretive learning: application to grammatical inference.
\newblock {\em Machine Learning\/}~{\em 94,\/}~1, 25--49.

\bibitem[\protect\citeauthoryear{Poxrucker, Bahle, and Lukowicz}{Poxrucker
  et~al\mbox{.}}{2014}]{poxrucker2014}
{\sc Poxrucker, A.}, {\sc Bahle, G.}, {\sc and} {\sc Lukowicz, P.} 2014.
\newblock Towards a real-world simulator for collaborative distributed learning
  in the scenario of urban mobility.
\newblock In {\em Proceedings of the Eighth IEEE International Conference on
  Self-Adaptive and Self-Organizing Systems Workshops}. IEEE Computer Society,
  44--48.

\bibitem[\protect\citeauthoryear{Ray}{Ray}{2009}]{ray2009nonmonotonic}
{\sc Ray, O.} 2009.
\newblock Nonmonotonic abductive inductive learning.
\newblock {\em Journal of Applied Logic\/}~{\em 7,\/}~3, 329--340.

\bibitem[\protect\citeauthoryear{Ray, Broda, and Russo}{Ray
  et~al\mbox{.}}{2003}]{ray2003hybrid}
{\sc Ray, O.}, {\sc Broda, K.}, {\sc and} {\sc Russo, A.} 2003.
\newblock Hybrid abductive inductive learning: A generalisation of progol.
\newblock In {\em Inductive Logic Programming}. Springer, 311--328.

\bibitem[\protect\citeauthoryear{Sakama}{Sakama}{2005}]{sakama2005}
{\sc Sakama, C.} 2005.
\newblock Induction from answer sets in nonmonotonic logic programs.
\newblock {\em ACM Transactions on Computational Logic (TOCL)\/}~{\em 6,\/}~2,
  203--231.

\bibitem[\protect\citeauthoryear{Sakama and Inoue}{Sakama and
  Inoue}{2009}]{Sakama2009}
{\sc Sakama, C.} {\sc and} {\sc Inoue, K.} 2009.
\newblock Brave induction: a logical framework for learning from incomplete
  information.
\newblock {\em Machine Learning\/}~{\em 76,\/}~1, 3--35.

\bibitem[\protect\citeauthoryear{Srinivasan}{Srinivasan}{2001}]{srinivasan2001aleph}
{\sc Srinivasan, A.} 2001.
\newblock The aleph manual.
\newblock {\em Machine Learning at the Computing Laboratory, Oxford
  University\/}.

\end{thebibliography}
\clearpage

\begin{appendix}
\section{Proofs}\label{sec:proofs}

\setcounter{fact}{0}
\setcounter{theorem}{1}

In this section, we give the proofs of the theorems in the main paper.  First,
we prove the preliminary lemma (Lemma~\ref{lem:context}). Really, this is a
corollary of the splitting set theorem~\cite{splitting}. $e_{U}(P, X)$ is the
partial evaluation of $P$ with respect to $X$ (over the atoms in $U$), which is
described in~\cite{splitting}.

\begin{lemma}

  For any program $P$ (consisting of normal rules, choice rules and constraints)
  and any set of pairs $S = \lbrace \langle C_1, \asp{a_1}\rangle, \ldots,
  \langle C_n, \asp{a_n}\rangle\rbrace$ such that none of the atoms $\asp{a_i}$
  appear in $P$ (or in any of the $C$'s) and each $\asp{a_i}$ atom is unique:
  $AS(P \cup \left\{ \asp{1 \lbrace a_1, \ldots, a_n \rbrace 1\ruleend}\right\}
  \cup \left\{ append(C_i, \asp{a_i}) \middle| \langle C_i, \asp{a_i}\rangle \in
  S\right\})= \left\{ A \cup \lbrace \asp{a_i}\rbrace \middle| A \in AS(P \cup
  C_i), \langle C_i, \asp{a_i}\rangle \in S\right\}$
\end{lemma}

\begin{proof}

  The answer sets of $\left\{ \asp{1 \lbrace a_1, \ldots, a_n \rbrace
  1\ruleend}\right\}$ are $\lbrace \asp{a_1}\rbrace, \ldots, \lbrace
  \asp{a_n}\rbrace$, hence by the splitting set theorem (using $U = \lbrace
  \asp{a_1}, \ldots, \asp{a_n}\rbrace$ as a splitting set):

  $AS(P \cup \left\{ \asp{1 \lbrace a_1, \ldots, a_n \rbrace 1\ruleend}\right\}
  \cup \left\{ append(C_i, \asp{a_i}) \middle| \langle C_i, \asp{a_i}\rangle \in
  S\right\})$\\
  \mbox{\hspace{10mm}}$= \left\{ A' \cup \lbrace \asp{a_j}\rbrace \middle|\begin{array}{c}
      \asp{a_j} \in \lbrace \asp{a_1}, \ldots, \asp{a_n}\rbrace\\
      A' \in AS(e_{U}(P \cup \lbrace append(C_i, \asp{a_i})\mid \langle C_i, \asp{a_i}\rangle \in A\rbrace, \lbrace \asp{a_j}\rbrace))\\
  \end{array}\right\}$\\
  \mbox{\hspace{10mm}}$= \left\{ A \cup \lbrace \asp{a_i}\rbrace\middle| A
    \in AS(P \cup C_i), \langle C_i, \asp{a_i}\rangle \in S\right\}$.
\end{proof}

\begin{theorem}\label{thm:complexity}
  The complexity of deciding whether an $ILP_{LOAS}^{context}$ task is satisfiable is $\Sigma^{P}_{2}$-complete.
\end{theorem}

\begin{proof}

  Deciding satisfiability for $ILP_{LOAS}$ is $\Sigma^{P}_2$-complete
  (\cite{ICLP15}). It is therefore sufficient to show that there is a
  polynomial mapping from $ILP_{LOAS}$ to $ILP_{LOAS}^{context}$ and a
  polynomial mapping from $ILP_{LOAS}^{context}$ to $ILP_{LOAS}$. The former is
  trivial (any $ILP_{LOAS}$ task can be mapped to the same task in
  $ILP_{LOAS}^{context}$ with empty contexts). The latter follows from
  theorem~\ref{thm:translation}.
\end{proof}

\begin{theorem}\label{thm:terminate}
  \algname terminates for any well defined $ILP_{LOAS}^{context}$ task.
\end{theorem}

\begin{proof}

  Assume that the task $T=\langle B, S_M, E\rangle$ is well defined. This means
  that $T_1 = \mathcal{T}_{LOAS}(T)$ is a well defined $ILP_{LOAS}$ task (every
  possible hypothesis has a finite grounding when combined with the background
  knowledge of $T_1$). Note that this also means that $T_2 =
  \mathcal{T}_{LOAS}(\langle B, S_M, Relevant\rangle)$ is well defined in each
  iteration as the size of the grounding of the background knowledge of $T_2$
  combined with each hypothesis will be smaller than or equal to the size of
  the background in $T_1$ (the background knowledge of $T_2$ is almost a subset
  of the background in $T_1$, other than the extra choice rule, which is
  smaller).

  The soundness of ILASP2~\cite{ICLP15} can be used to show that $H$ will
  always cover every example in $Relevant$; hence, at each step $re$ must be an
  example which is in $E$ but not in $Relevant$.  As there are a finite number
  of examples in $E$, this means there can only be a finite number of
  iterations; hence, it remains to show that each iteration terminates.
  This is the case because, as $\mathcal{T}_{LOAS}(\langle B, S_M,
  Relevant\rangle)$ is well defined, the call to ILASP2 terminates
  (\cite{ICLP15}) and $findRelevantExample$ terminates (\ref{sec:meta}).
\end{proof}

\begin{theorem}
  \algname is sound for any well defined $ILP_{LOAS}^{context}$ task, and
  returns an optimal solution if one exists.
\end{theorem}

\begin{proof}
  If the \algname algorithm returns a hypothesis then the while loop must
  terminate. For this to happen $findRelevantExample$ must return \texttt{nil}.
  This means that $H$ must cover every example in $E$. Hence \algname is
  sound.
  As the algorithm terminates (see Theorem~\ref{thm:terminate}), the only way
  for a solution not to be returned is when $ILASP2$ returns \texttt{nil}.
  Since $ILASP2$ is complete~\cite{ICLP15}, this is only possible when $\langle
  B, S_M, Relevant\rangle$ is unsatisfiable. But if $\langle B, S_M,
  Relevant\rangle$ is unsatisfiable then so is $\langle B, S_M, E\rangle$.

  It remains to show that when a solution is returned, it is an optimal
  solution. Any solution $H$ returned must be an optimal solution of $\langle
  B, S_M, Relevant\rangle$, (as ILASP2\break returns an optimal solution).  As it
  must also be a solution of $\langle B, S_M, E\rangle$, it must be an optimal
  solution (any shorter solution would be a solution of $\langle B, S_M,
  Relevant\rangle$, contradicting that $H$ is an optimal solution for $\langle
  B, S_M, Relevant\rangle$).
\end{proof}

\section{$findRelevantExamples$}
\label{sec:meta}

In this section, we describe (and prove the correctness of) the
$findRelevantExamples$ method which was omitted from the main paper.  The
method uses a meta encoding in ASP. Given a learning task and a hypothesis
from the hypothesis space, this meta encoding is used to compute the set
of examples that are covered and the set that are not covered. The meta
encoding is formalised in definition~\ref{def:meta}, but we first introduce
some notation in order to simplify the main definition. Some definitions are
similar to those used in the ILASP2 meta representation~\cite{ICLP15}.

\begin{definition}
  For any ASP program $P$ and predicate name $\asp{pred}$, $reify(P,
  \asp{pred})$ denotes the program constructed by replacing every atom $\asp{a}
  \in P'$ (where $P'$ is $P$ with the weak constraints removed) by
  $\asp{pred(a)}$.  We use the same notation for sets of literals/partial
  interpretations, so for a set $S$: $reify(S, pred) = \lbrace \asp{pred(atom)}
  : \asp{atom} \in S \rbrace$.
\end{definition}

Definition~\ref{def:weak_rep} formalises the way we represent weak
constraints in our meta encoding. We use this representation to check whether
ordering examples are covered. We use $\asp{as1}$ and $\asp{as2}$ to represent
the atoms in two answer sets ($\asp{as1}$ and $\asp{as2}$ occur elsewhere in
our encoding). The $\asp{w}$ atoms are then used to capture the penalties paid
by each answer set at each level.

\begin{definition}\label{def:weak_rep}
  For any ASP program $P$, we write $weak(P)$ to
  mean the program constructed from the weak constraints in $P$, translating
  each weak constraint $\asp{:\sim b_1, \ldots, b_m,}$\break $\asp{\naf
  b_{m+1}, \ldots, \naf b_n\ruleend[lev@wt, t_1, \ldots, t_{k}]}$ to the rules:

    $\left\{\begin{array}{l}
      \asp{w(wt, lev, terms(t_1, \ldots, t_k), as1) \codeif as1(b_1), \ldots, as1(b_m),}\\
      \asp{\mbox{\hspace{56mm}}\naf as1(b_{m+1}), \ldots, \naf as1(b_n)\ruleend}\\
      \asp{w(wt, lev, terms(t_1, \ldots, t_k), as2) \codeif as2(b_1), \ldots, as2(b_m),}\\
      \asp{\mbox{\hspace{56mm}}\naf as2(b_{m+1}), \ldots, \naf as2(b_n)\ruleend}
    \end{array}\right\}$
\end{definition}

We now introduce a simplified version of the ASP program fragment which is used
by ILASP2 to check whether one answer set dominates another. This is used in
determining whether an ordering example is covered by a hypothesis. This makes
use of the $\asp{w}$ atoms which are generated by the $\asp{w}$ rules in
definition~\ref{def:weak_rep}, and captures the definition of dominates given
in Section~\ref{sec:background}.

\begin{definition}

  $dominates$ is the program:

\begin{math}
\left\{\begin{array}{l}
\asp{dom\_lv(L) \codeif lv(L), \texttt{\#sum}\lbrace w(W,L,A, as1)=W, w(W,L,A, as2)=-W\rbrace < 0\ruleend}\\
\asp{non\_dom\_lv(L) \codeif lv(L), \texttt{\#sum}\lbrace w(W,L,A,as2)=W, w(W,L,A,as1)=-W\rbrace < 0\ruleend}\\
\asp{non\_bef(L) \codeif lv(L), lv(L2), L < L2, non\_dom\_lv(L2)\ruleend}\\
\asp{dominated \codeif dom\_lv(L), \naf non\_bef(L)\ruleend}\\
\end{array}\right\}
\end{math}

\end{definition}

In~\cite{ICLP15}, multiple instances of $dominates$ were included in the same
meta encoding, and hence the program was slightly more complicated in
order to track the different instances. The main structure of the program is
the same however, and hence the same results apply. The result we need for this
paper is proven (for the more general program) in~\cite{ILASP2Proof} and is
given by Lemma~\ref{lem:dominates}.

\begin{lemma}\label{lem:dominates}
  Let $I_1$ and $I_2$ be interpretations, $P$ be an ASP program and $L$ be the
  set of levels used in the weak constraints in $P$. The unique answer set of
%
  $
      dominates\cup
      \lbrace \asp{lv(l)\ruleend} \mid l \in L\rbrace\cup
      weak(P)
      \cup
          reify(I_1, \asp{as1}) \cup
          reify(I_1, \asp{as2})
          $
          contains the atom $\mathtt{dominated}$ if and only if $I_1$
  dominates $I_2$ wrt the weak constraints in $P$.
\end{lemma}

Definition~\ref{def:meta} captures the meta encoding we use in
$findRelevantExamples$. This encoding is made of 6 components. $\mathcal{R}_1$
captures the background knowledge and hypothesis -- by reifying $B\cup H$, the
$\asp{as1}$ and $\asp{as2}$ atoms represent two answer sets $A_1$ and
$A_2$, and the $dominates$ program (together with $weak(B \cup H)$ and the
priority levels) checks whether $A_1$ dominates $A_2$. The programs
$\mathcal{R}_2$ to $\mathcal{R}_5$ are used to check whether each type of
example is covered. These programs make use of the predicate $\asp{test\_on}$
of arity 2 and the $\asp{test}$ predicate of arity 1. The meaning of
$\asp{test(ex_{id})}$ is that the example $ex$ should be tested. There is a
choice rule in $\mathcal{R}_6$ to say that each example should be tested.  For
the positive and negative examples, this means that they should be tested on
$\asp{as1}$ (meaning to check whether it is possible that an answer set of $B
\cup H$ extends this example). For an ordering example $\langle ex_1,
ex_2\rangle$ it is slightly more involved: $ex_1$ should be tested on
$\asp{as1}$ and $ex_2$ should be tested on $\asp{as2}$ (and the ordering should
be checked).

\begin{definition}\label{def:meta}

Let $T$ be the $ILP_{LOAS}^{context}$ task $\langle B, S_M, \langle E^{+}, E^{-},
O^{b}, O^{c}\rangle\rangle$ and $H$ be a hypothesis such that $H \subseteq S_M$. Let
$L$ be the set of all priority levels in $B \cup H$ $\mathcal{R}(T, H)$ is the
ASP program $\mathcal{R}_1(B \cup H) \cup \mathcal{R}_2(E^{+}) \cup
\mathcal{R}_3(E^{-}) \cup \mathcal{R}_4(O^{b}) \cup \mathcal{R}_5(O^{c}) \cup
\mathcal{R}_6(E^{+}\cup E^{-}, O^{b} \cup O^{c})$, where the individual
components are as follows:

\begin{itemize}
  \item $\mathcal{R}_1(B\cup H) = reify(B \cup H, \asp{as1}) \cup reify(B \cup H, \asp{as2})\cup
    weak(B\cup H) \cup \lbrace \asp{lv(l)\ruleend}\mid \asp{l} \in L\rbrace \cup dominates$
  \item $\mathcal{R}_2(E^{+}) = \left\{
    \begin{array}{l}
      \asp{cov(as1) \codeif test\_on(ex_{id}, as1),}\\
      \mbox{\hspace{5mm}}\asp{as1(e^{inc}_1),\ldots,as1(e^{inc}_m),}\\
      \mbox{\hspace{5mm}}\asp{\naf as1(e^{exc}_1),\ldots,\naf as1(e^{exc}_n)}\\
      \asp{cov(as2) \codeif test\_on(ex_{id}, as2),}\\
      \mbox{\hspace{5mm}}\asp{as2(e^{inc}_1),\ldots,as2(e^{inc}_m),}\\
      \mbox{\hspace{5mm}}\asp{\naf as2(e^{exc}_1),\ldots,\naf as2(e^{exc}_n)}\\
      \asp{\codeif \naf cov(as1), test\_on(ex_{id}, as1)\ruleend}\\
      \asp{\codeif \naf cov(as2), test\_on(ex_{id}, as2)\ruleend}\\
      append(reify(C, \asp{as1}), \asp{test\_on(ex_{id}, as1)})\\
      append(reify(C, \asp{as2}), \asp{test\_on(ex_{id}, as2)})\\
    \end{array}\middle|
    \begin{array}{c}
      ex \in E^{+},\\
      ex = \langle e, C \rangle,\\
      e = \langle \lbrace \asp{e^{i}_1,\ldots,e^{i}_m}\rbrace, \lbrace \asp{e^{e}_1,\ldots,e^{e}_n}\rbrace\rangle
    \end{array}\right\}$
  \item $\mathcal{R}_3(E^{-}) = \left\{
    \begin{array}{l}
    \asp{violated \codeif test\_on(ex_{id}, as1),}\\
      \mbox{\hspace{5mm}}\asp{as1(e^{inc}_1),\ldots, as1(e^{inc}_m),}\\
      \mbox{\hspace{5mm}}\asp{\naf as1(e^{exc}_1),\ldots,\naf as1(e^{exc}_n)\ruleend}\\
      append(reify(C, \asp{as1}), \asp{test\_on(ex_{id}, as1)})\\
      \asp{\codeif \naf violated, test\_on(ex_{id}, as1)\ruleend}
    \end{array}\middle|
    \begin{array}{c}
      ex \in E^{-},\\
      ex = \langle e, C \rangle,\\
      e = \langle \lbrace \asp{e^{i}_1,\ldots,e^{i}_m}\rbrace, \lbrace \asp{e^{e}_1,\ldots,e^{e}_n}\rbrace\rangle
    \end{array}\right\}$
  \item $\mathcal{R}_4(O^{b}) = \left\{
    \begin{array}{l}
      \asp{\codeif test(o_{id}), \naf dominated\ruleend}
    \end{array}\middle|
    \begin{array}{c}
      o \in O^{b}
    \end{array}\right\}$
  \item $\mathcal{R}_5(O^{c}) = \left\{
    \begin{array}{l}
      \asp{\codeif test(o_{id}), dominated\ruleend}
    \end{array}\middle|
    \begin{array}{c}
      o \in O^{c}
    \end{array}\right\}$
  \item $\mathcal{R}_6(\lbrace ex_1, \ldots ex_m \rbrace, \lbrace o_1, \ldots o_n \rbrace) = \left\{
    \begin{array}{l}
      \asp{1 \lbrace test(ex_1), \ldots, test(ex_m), test(o_1), \ldots, test(o_n) \rbrace 1\ruleend}
    \end{array}\right\}\\$
    $\mbox{\hspace{50mm}}\cup\left\{
    \begin{array}{l}
      \asp{test\_on(ex_i, \asp{as1})\codeif test(ex_i)\ruleend}
    \end{array}\middle|\begin{array}{l}
      ex_i \in \lbrace ex_1, \ldots, ex_m\rbrace
    \end{array}\right\}\\$
    $\mbox{\hspace{50mm}}\cup\left\{
    \begin{array}{l}
      \asp{test\_on(ex_1, \asp{as1})\codeif test(o_i)\ruleend}\\
      \asp{test\_on(ex_2, \asp{as2})\codeif test(o_i)\ruleend}
    \end{array}\middle|\begin{array}{c}
      o_i \in \lbrace o_1, \ldots, o_n\rbrace\\
      o_i = \langle ex_1, ex_2\rangle
    \end{array}\right\}\\$
\end{itemize}

\end{definition}

\begin{theorem}\label{thm:fre}
  Let $T$ be any $ILP_{LOAS}^{context}$ task and $H$ be any subset of the
  hypothesis space.

  \begin{enumerate}
    \item $\forall ex \in E^{+}$, $\exists A \in AS(\mathcal{R}(T, H))$ st
      $\asp{test(ex_{id})} \in A$ iff $H$ covers $ex$.
    \item $\forall ex \in E^{-}$, $\exists A \in AS(\mathcal{R}(T, H))$ st
      $\asp{test(ex_{id})} \in A$ iff $H$ does not cover $ex$.
    \item $\forall o \in O^{b}$, $\exists A \in AS(\mathcal{R}(T, H))$ st
      $\asp{test(o_{id})} \in A$ iff $H$ bravely respects $o$.
    \item $\forall o \in O^{c}$, $\exists A \in AS(\mathcal{R}(T, H))$ st
      $\asp{test(o_{id})} \in A$ iff $H$ does not cautiously
      respect $o$.
  \end{enumerate}
\end{theorem}

\begin{proof}
  \begin{enumerate}
    \item Let $ex = \langle e, C\rangle$ be a CDPI in $E^{+}$ st $e = \langle
      \lbrace \asp{e^{i}_1}, \ldots \asp{e^{i}_m}\rbrace,$ $\lbrace
      \asp{e^{e}_1}, \ldots, \asp{e^{e}_n}\rbrace\rangle$.

      \noindent $H$ covers $ex\Leftrightarrow \exists A \in AS(B \cup H \cup
      C)$ st $A$ extends $e$

      \leftskip=0mm

      $\Leftrightarrow \exists A \in AS(reify(B \cup H \cup C, \asp{as1}))$ st
      $A$ extends $reify(e, \asp{as1})$

      $\Leftrightarrow reify(B \cup H \cup C, \asp{as1}) \cup \left\{
        \begin{array}{l}
        \asp{cov(as1) \codeif as1(e_1), \ldots, as1(e_m),}\\
        \asp{\hspace{20mm}\naf as1(e_1), \ldots, \naf as1(e_n)\ruleend}\\
          \asp{\codeif \naf cov(as1)\ruleend}
        \end{array}
      \right\}$ is satisfiable (we refer to this program as $P_1$ later in the
      proof).

      $\Leftrightarrow reify(B \cup H, \asp{as1}) \cup append(reify(C,
      \asp{as1}), \asp{test\_on(ex_{id}, as1)})\\
      \mbox{\hspace{10mm}}\cup \left\{\begin{array}{l}
        \asp{cov(as1) \codeif test\_on(ex_{id}, as1),}\\
          \mbox{\hspace{5mm}}\asp{as1(e_1), \ldots, as1(e_m),}\\
          \mbox{\hspace{5mm}}\asp{\naf as1(e_1), \ldots, \naf as1(e_n)\ruleend}\\
          \asp{\codeif \naf cov(as1), test\_on(ex_{id}, as1)\ruleend}
      \end{array}\right\} \cup \mathcal{R}_6(E^{+}\cup
      E^{-}, O^{b}\cup O^{c})$\break has an answer set which contains
      $\asp{test(ex_{id})}$ (we refer to this program as $P_2$). This follows
      from the splitting set theorem, using the atoms in
      $\mathcal{R}_6(E^{+}\cup E^{-}, O^{b}\cup O^{c})$ as a splitting set --
      $\lbrace \asp{test(ex_{id})}, \asp{test\_on(ex_{id}, as1)}\rbrace$ is an
      answer set of the bottom program, leading to $P_1$ as the partially
      evaluated top program

      $\Leftrightarrow \mathcal{R}(T, H)$ has an answer set which contains
      $\asp{test(ex_{id})}$. Again, this is by the splitting set theorem,
      using the atoms in $\mathcal{R}_6(E^{+}\cup E^{-}, O^{b}\cup O^{c})$ as a
      splitting set, as $P_2 \subseteq \mathcal{R}(T,H)$ and each of the extra
      rules in $\mathcal{R}(T,H)$ which are not in $P_2$ contain a
      $\asp{test\_on}$ or $\asp{test}$ atom in the body that is not in the
      answer set $\lbrace \asp{test(ex_{id})}, \asp{test\_on(ex_{id},
      as1)}\rbrace$ and hence they are removed from the partially evaluated top
      program.

      \leftskip=0mm

    \item Let $ex = \langle e, C\rangle$ be a CDPI in $E^{-}$ st $e = \langle
      \lbrace \asp{e^{i}_1}, \ldots \asp{e^{i}_m}\rbrace,$ $\lbrace
      \asp{e^{e}_1}, \ldots, \asp{e^{e}_n}\rbrace\rangle$.

      \noindent $H$ does not cover $ex\Leftrightarrow
      \exists A \in AS(B \cup H \cup C)$ st $A$ extends $e$

      \leftskip=5mm

      $\Leftrightarrow \exists A \in AS(reify(B \cup H \cup C, \asp{as1}))$ st
      $A$ extends $reify(e, \asp{as1})$

      $\Leftrightarrow reify(B \cup H \cup C, \asp{as1}) \cup \left\{
        \begin{array}{l}
          \asp{violated \codeif as1(e_1), \ldots, as1(e_m),}\\
          \mbox{\hspace{5mm}}\asp{\naf as1(e_1), \ldots, \naf as1(e_n)\ruleend}\\
          \asp{\codeif \naf violated\ruleend}
        \end{array}
      \right\}$ is satisfiable (we refer to this program as $P_3$ later in the
      proof)

      $\Leftrightarrow reify(B \cup H, \asp{as1}) \cup append(reify(C,
      \asp{as1}), \asp{test\_on(ex_{id}, as1)})\\
      \mbox{\hspace{10mm}}\cup \left\{\begin{array}{l}
        \asp{violated \codeif test\_on(ex_{id}, as1),}\\
        \mbox{\hspace{5mm}}\asp{as1(e_1), \ldots, as1(e_m),}\\
        \mbox{\hspace{5mm}}\asp{\naf as1(e_1), \ldots, \naf as1(e_n)\ruleend}\\
          \asp{\codeif \naf violated, test\_on(ex_{id}, as1)\ruleend}
      \end{array}\right\} \cup \mathcal{R}_6(E^{+}\cup
      E^{-}, O^{b}\cup O^{c})$\break has an answer set which contains
      $\asp{test(ex_{id})}$ (we refer to this program as $P_4$). This follows by
      the splitting set theorem, using the atoms in $\mathcal{R}_6(E^{+}\cup
      E^{-}, O^{b}\cup O^{c})$ as a splitting set, $\lbrace \asp{test(ex_{id})},
      \asp{test\_on(ex_{id}, as1)}\rbrace$ is an answer set of the bottom
      program, leading to $P_3$ as the partially evaluated top program.

      $\Leftrightarrow \mathcal{R}(T, H)$ has an answer set which contains
        $\asp{test(ex_{id})}$. Again, this is by the splitting set theorem, using
      the atoms in $\mathcal{R}_6(E^{+}\cup E^{-}, O^{b}\cup O^{c})$ as a splitting
      set, as $P_4 \subseteq \mathcal{R}(T, H)$ and each of the extra rules in
      $\mathcal{R}(T, H)$ which are not in $P_4$ contain a $\asp{test\_on}$ or
      $\asp{test}$ atom in the body that is not in the answer set $\lbrace
      \asp{test(ex_{id})}, \asp{test\_on(ex_{id}, as1)}\rbrace$ and hence they
      are removed from the partially evaluated top program.

      \leftskip=0mm

    \item Let $o = \langle ex1, ex2\rangle$ be a CDOE in $O^{b}$ st $ex1 =
      \langle e1, C_1\rangle$, $ex2 = \langle e2, C_2\rangle$, $e_1 = \langle
      \lbrace \asp{e1^{i}_1}, \ldots, \asp{e1^{i}_m}\rbrace, \lbrace
      \asp{e1^{e}_1}, \ldots, \asp{e1^{e}_n}\rbrace\rangle$ and $e_2 = \langle
      \lbrace \asp{e2^{i}_1}, \ldots, \asp{e2^{i}_j}\rbrace, \lbrace
      \asp{e2^{e}_1}, \ldots, \asp{e2^{e}_k}\rbrace\rangle$.

      \noindent $H$ bravely respects $o \Leftrightarrow
      \exists A_1 \in AS(B \cup H \cup C_1), \exists A_2 \in AS(B \cup H \cup
      C_2)$ st $A_1$ extends $e_1$, $A_2$ extends $e_2$ and $A_1 \prec_{B \cup
      H} A_2$

      \leftskip=5mm

      $\Leftrightarrow \exists A_1 \in AS(reify(B \cup H \cup C_1, \asp{as1})), \exists A_2 \in
      AS(reify(B \cup H \cup C_2, \asp{as2}))$ st $A_1$ extends $reify(e_1,
      \asp{as1})$, $A_2$ extends $reify(e_2, \asp{as2})$ and $\asp{dominated}$
      is in the unique answer set of $A_1\cup A_2\cup weak(B \cup H) \cup
      \lbrace \asp{lv(l)\ruleend} \mid \asp{l} \in L\rbrace \cup dominates$
      (by Lemma~\ref{lem:dominates})

      $\Leftrightarrow reify(B \cup H \cup C_1, \asp{as1}) \cup
      reify(B \cup H \cup C_2, \asp{as2}) \cup weak(B \cup
      H) \cup \lbrace \asp{lv(l)\ruleend}
      \mid \asp{l} \in L\rbrace \cup dominates)$\\
      \mbox{\hspace{20mm}}$\cup \left\{\begin{array}{l}
          \asp{cov(as1) \codeif as1(e1^{i}_1), \ldots, as1(e1^{i}_m),}\\
            \mbox{\hspace{5mm}}\asp{\naf as1(e1^{e}_1), \ldots, \naf as1(e1^{e}_n)\ruleend}\\
          \asp{\codeif \naf cov(as1)\ruleend}\\
          \asp{cov(as2) \codeif as2(e2^{i}_1), \ldots, as2(e2^{i}_j),}\\
            \mbox{\hspace{5mm}}\asp{\naf as2(e2^{e}_1), \ldots, \naf as2(e2^{e}_k)\ruleend}\\
          \asp{\codeif \naf cov(as2)\ruleend}\\
          \asp{\codeif \naf dominated\ruleend}\\
      \end{array}\right\}$ is satisfiable (we refer to this program as $P_5$
      later in the proof)

      \noindent
      $\Leftrightarrow reify(B \cup H \cup C_1, \asp{as1}) \cup
      reify(B \cup H \cup C_2, \asp{as2}) \cup weak(B \cup H) \cup \lbrace
      \asp{lv(l)\ruleend}
      \mid \asp{l} \in L\rbrace \cup dominates)$\\
      \mbox{\hspace{20mm}}$\cup \left\{\begin{array}{l}
          \asp{cov(as1) \codeif test\_on(ex1_{id}, as2), as1(e1^{i}_1), \ldots, as1(e1^{i}_m),}\\
            \mbox{\hspace{5mm}}\asp{\naf as1(e1^{e}_1), \ldots, \naf as1(e1^{e}_n)\ruleend}\\
          \asp{\codeif \naf test\_on(ex1_{id}, as1), cov(as1)\ruleend}\\
          \asp{cov(as2) \codeif test\_on(ex2_{id}, as2), as2(e2^{i}_1), \ldots, as2(e2^{i}_j),}\\
            \mbox{\hspace{5mm}}\asp{\naf as2(e2^{e}_1), \ldots, \naf as2(e2^{e}_k)\ruleend}\\
          \asp{\codeif test\_on(ex2_{id}, as2), \naf cov(as2)\ruleend}\\
          \asp{\codeif test(o_{id}), \naf dominated\ruleend}\\
      \end{array}\right\}$\\
      \mbox{\hspace{20mm}}$\cup \mathcal{R}_6(E^{+}\cup E^{-}, O^{b} \cup O^{c})$\\
      has an answer set which contains $\asp{test(o_{id})}$ (we refer to this
      program as $P_6$). This follows by the splitting set theorem, using the
      atoms in $\mathcal{R}_6(E^{+}\cup E^{-}, O^{b}\cup O^{c})$ as a splitting
      set, $\lbrace \asp{test(o_{id})}, \asp{test\_on(ex1_{id}, as1)},$
      $\asp{test\_on(ex2_{id}, as2)}\rbrace$ is an answer set of the bottom
      program, leading to $P_5$ as the partially evaluated top program

      \noindent
      $\Leftrightarrow \mathcal{R}(T, H)$ has an answer set which contains
      $\asp{test(o_{id})}$. Again, this is by the splitting set theorem,
      using the atoms in $\mathcal{R}_6(E^{+}\cup E^{-}, O^{b}\cup O^{c})$ as a
      splitting set, as $P_6 \subseteq \mathcal{R}_6(T, H)$ and each of the
      extra rules which are in $\mathcal{R}_6(T, H)$ but not in $P_6$ contain a
      $\asp{test\_on}$ or $\asp{test}$ atom which is not in the answer set
      $\lbrace \asp{test(o_{id})}, \asp{test\_on(ex1_{id}, as1)},$
      $\asp{test\_on(ex2_{id}, as2)}\rbrace$ and hence they are removed from the
      partially evaluated top program

    \leftskip=0mm

  \item Let $o = \langle ex1, ex2\rangle$ be a CDOE in $O^{c}$ st $ex1 =
    \langle e1, C_1\rangle$, $ex2 = \langle e2, C_2\rangle$,  $e_1 = \langle
    \lbrace \asp{e1^{i}_1}, \ldots, \asp{e1^{i}_m}\rbrace, \lbrace
    \asp{e1^{e}_1}, \ldots, \asp{e1^{e}_n}\rbrace\rangle$ and $e_2 = \langle
    \lbrace \asp{e2^{i}_1}, \ldots, \asp{e2^{i}_j}\rbrace, \lbrace
    \asp{e2^{e}_1}, \ldots, \asp{e2^{e}_k}\rbrace\rangle$

      \noindent
      $H$ does not cautiously respect $o \Leftrightarrow \exists A_1 \in AS(B
      \cup H \cup C_1), \exists A_2 \in AS(B \cup H \cup C_2)$ st $A_1$ extends
      $e_1$, $A_2$ extends $e_2$ and $A_1 \not\prec_{B \cup H} A_2$

    \leftskip=5mm

      $\Leftrightarrow \exists A_1 \in AS(reify(B \cup H \cup C_1, \asp{as1})),
      \exists A_2 \in AS(reify(B \cup H \cup C_2, \asp{as2}))$ st $A_1$ extends
      $reify(e_1, \asp{as1})$, $A_2$ extends $reify(e_2, \asp{as2})$ and
      $\asp{dominated}$ is not in the unique answer set of $A_1\cup A_2\cup
      weak(B \cup H) \cup \lbrace \asp{lv(l)\ruleend} \mid \asp{l} \in L\rbrace
      \cup dominates$ (by Lemma~\ref{lem:dominates})

      $\Leftrightarrow reify(B \cup H \cup C_1, \asp{as1}) \cup
      reify(B \cup H \cup C_2, \asp{as2}) \cup weak(B \cup
      H) \cup \lbrace \asp{lv(l)\ruleend}
      \mid \asp{l} \in L\rbrace \cup dominates)$\\
      \mbox{\hspace{20mm}}$\cup \left\{\begin{array}{l}
          \asp{cov(as1) \codeif as1(e1^{i}_1), \ldots, as1(e1^{i}_m),}\\
            \mbox{\hspace{5mm}}\asp{\naf as1(e1^{e}_1), \ldots, \naf as1(e1^{e}_n)\ruleend}\\
          \asp{\codeif \naf cov(as1)\ruleend}\\
          \asp{cov(as2) \codeif as2(e2^{i}_1), \ldots, as2(e2^{i}_j),}\\
            \mbox{\hspace{5mm}}\asp{\naf as2(e2^{e}_1), \ldots, \naf as2(e2^{e}_k)\ruleend}\\
          \asp{\codeif \naf cov(as2)\ruleend}\\
          \asp{\codeif dominated\ruleend}\\
      \end{array}\right\}$ is satisfiable (we refer to this program as $P_7$
      later in the proof)

      $\Leftrightarrow reify(B \cup H \cup C_1, \asp{as1}) \cup
      reify(B \cup H \cup C_2, \asp{as2}) \cup weak(B \cup H) \cup \lbrace
      \asp{lv(l)\ruleend}
      \mid \asp{l} \in L\rbrace \cup dominates)$\\
      \mbox{\hspace{20mm}}$\cup \left\{\begin{array}{l}
          \asp{cov(as1) \codeif test\_on(ex1_{id}, as2), as1(e1^{i}_1), \ldots, as1(e1^{i}_m),}\\
            \mbox{\hspace{5mm}}\asp{\naf as1(e1^{e}_1), \ldots, \naf as1(e1^{e}_n)\ruleend}\\
          \asp{\codeif \naf test\_on(ex1_{id}, as1), cov(as1)\ruleend}\\
          \asp{cov(as2) \codeif test\_on(ex2_{id}, as2), as2(e2^{i}_1), \ldots, as2(e2^{i}_j),}\\
            \mbox{\hspace{5mm}}\asp{\naf as2(e2^{e}_1), \ldots, \naf as2(e2^{e}_k)\ruleend}\\
          \asp{\codeif test\_on(ex2_{id}, as2), \naf cov(as2)\ruleend}\\
          \asp{\codeif test(o_{id}), dominated\ruleend}\\
      \end{array}\right\}$\\
      \mbox{\hspace{20mm}}$\cup \mathcal{R}_6(E^{+}\cup E^{-}, O^{b} \cup O^{c})$\\
      has an answer set which contains the atom $\asp{test(o_{id})}$ (we refer
      to this program as $P_8$). This follows from the splitting set theorem,
      using the atoms in $\mathcal{R}_6(E^{+}\cup E^{-}, O^{b}\cup O^{c})$ as a
      splitting set, $\lbrace \asp{test(o_{id})},
      \asp{test\_on(ex1_{id}, as1)}, \asp{test\_on(ex2_{id}, as2)}\rbrace$ is an
      answer set of the bottom program, leading to $P_7$ as the partially
      evaluated top program

      $\Leftrightarrow \mathcal{R}(T, H)$ has an answer set which contains
      $\asp{test(o_{id})}$. Again, this is by the splitting set theorem,
      using the atoms in $\mathcal{R}_6(E^{+}\cup E^{-}, O^{b}\cup O^{c})$ as a
      splitting set, as $P_8 \subseteq \mathcal{R}_6(T, H)$ and each of the
      extra rules which are in $\mathcal{R}_6(T, H)$ but not in $P_8$ contain a
      $\asp{test\_on}$ or $\asp{test}$ atom which is not in the answer set
      $\lbrace \asp{test(o_{id})}, \asp{test\_on(ex1_{id}, as1)},$ $
      \asp{test\_on(ex2_{id}, as2)}\rbrace$ and hence they are removed from the
      partially evaluated top program.

    \leftskip=0mm

  \end{enumerate}
\end{proof}

$findRelevantExamples(T, H)$ works by constructing $\mathcal{R}(T, H)$ and
computing its answer sets. For each example $ex$, whether of not $ex$ is covered
by $T$ can be computed from the answer sets, using the results in
Theorem~\ref{thm:fre}.  The first example which is not covered is returned. If
no such example is found, $\asp{nil}$ is returned. The correctness of
$findRelevantExamples$ follows directly from Theorem~\ref{thm:fre}. If the task
$T$ is well defined then $\mathcal{R}(T,H)$ will ground finitely (and have a
finite number of answer sets), and therefore solving $\mathcal{R}(T,H)$ for
answer sets will terminate in a finite time; hence as there are a finite number
of examples, $findRelevantExamples$ will terminate in a finite time.


%
%

\end{appendix}

\end{document}